\documentclass{article}

\usepackage{microtype}
\usepackage{graphicx}
\usepackage{subfigure}
\usepackage{booktabs} % for professional tables

\usepackage{hyperref}

\usepackage[accepted]{icml2024}

\usepackage{amsmath}
\usepackage{amssymb}
\usepackage{mathtools}
\usepackage{amsthm}
\usepackage{multicol}
\usepackage{multirow}
\usepackage[capitalize,noabbrev]{cleveref}

\theoremstyle{plain}
\newtheorem{theorem}{Theorem}[section]

\newtheorem{lemma}[theorem]{Lemma}

\theoremstyle{definition}
\newtheorem{definition}[theorem]{Definition}
\newtheorem{assumption}[theorem]{Assumption}
\theoremstyle{remark}

\usepackage[textsize=tiny]{todonotes}

\icmltitlerunning{Learning High-quality Model and Policy from Low-quality Offline Visual Datasets}

\begin{document}

\twocolumn[
\icmltitle{SeMOPO: Learning High-quality Model and Policy from Low-quality \\Offline Visual Datasets}



\icmlsetsymbol{equal}{*}

\begin{icmlauthorlist}
\icmlauthor{Shenghua Wan}{nju,lamda}
\icmlauthor{Ziyuan Chen}{pku}
\icmlauthor{Le Gan}{nju,lamda}
\icmlauthor{Shuai Feng}{bit}
\icmlauthor{De-Chuan Zhan}{nju,lamda}
\end{icmlauthorlist}

\icmlaffiliation{nju}{School of Artificial Intelligence, Nanjing University, China} 
\icmlaffiliation{lamda}{National Key Laboratory for Novel Software Technology, Nanjing University, China}
\icmlaffiliation{pku}{School of Mathematical Sciences, Center for Statistical Science, Peking University, Beijing, China}
\icmlaffiliation{bit}{School of Cyberspace Science and Technology, Beijing Institute of Technology, Beijing, China}

\icmlcorrespondingauthor{De-Chuan Zhan}{zhandc@nju.edu.cn}

\icmlkeywords{Machine Learning, ICML}

\vskip 0.3in
]



\printAffiliationsAndNotice{}  % leave blank if no need to mention equal contribution

\begin{abstract}
Model-based offline reinforcement Learning (RL) is a promising approach that leverages existing data effectively in many real-world applications, especially those involving high-dimensional inputs like images and videos. To alleviate the distribution shift issue in offline RL, existing model-based methods heavily rely on the uncertainty of learned dynamics. However, the model uncertainty estimation becomes significantly biased when observations contain complex distractors with non-trivial dynamics. To address this challenge, we propose a new approach - \emph{Separated Model-based Offline Policy Optimization} (SeMOPO) - decomposing latent states into endogenous and exogenous parts via conservative sampling and estimating model uncertainty on the endogenous states only. We provide a theoretical guarantee of model uncertainty and performance bound of SeMOPO. To assess the efficacy, we construct the Low-Quality Vision Deep Data-Driven Datasets for RL (LQV-D4RL), where the data are collected by non-expert policy and the observations include moving distractors. Experimental results show that our method substantially outperforms all baseline methods, and further analytical experiments validate the critical designs in our method. The project website is \href{https://sites.google.com/view/semopo}{https://sites.google.com/view/semopo}.
\end{abstract}

\section{Introduction}
\label{sec:intro}

Offline reinforcement learning (RL)~\cite{lange2012batch,Levine2020OfflineRL}, which learns policies from fixed datasets without the need for costly interactions with online environments, has been increasingly applied in real-world tasks such as drug discovery~\cite{Designbench22} and autonomous driving~\cite{offlinerl_autodriving2021,offlinerl_autodriving2023}. Offline RL saves the cost of interacting with the environment and improves sample efficiency. Real-world RL datasets typically exhibit two main characteristics: (1) they are often collected by non-expert or random policies~\cite{Jin2020IsPP,Rashidinejad2021BridgingOR}, and (2) they originate from real environments with high-dimensional observations, such as images or videos~\cite{rafailov2021offline,offline_survey2022}, containing complex noise like moving backgrounds~\cite{2022vd4rl}. Learning high-quality policies from low-quality datasets poses a significant challenge.

Concerning the first characteristic, previous research has demonstrated the efficacy of offline RL with highly sub-optimal or random datasets~\cite{Jin2020IsPP,Rashidinejad2021BridgingOR,offline2BC2022}. Regarding the second characteristic, some studies have employed model-based RL (MBRL) methods to address the challenges of high-dimensional inputs. MBRL learns a low-dimensional surrogate model of the high-dimensional environment~\cite{world_models18,planet19}, allowing the agent to interact with this model to gather additional trajectories in the low-dimensional state space. Offline MBRL methods improve sample efficiency and reduce storage costs. However, an unavoidable gap exists between the actual environment and the learned model from offline datasets. Offline MBRL methods~\cite{yu2020mopo} mitigate the distribution shift problem by incorporating the model prediction's uncertainty as a penalty in the reward function.

In real-world decision-making scenarios, uncertainty arises not only from task-relevant dynamics but also from irrelevant distractors in observations, such as moving backgrounds. However, previous offline visual RL works~\cite{yu2020mopo,rafailov2021offline,2022vd4rl} do not differentiate between these two types of uncertainty. If both are treated as model uncertainty in offline policy training and used as a penalty in the reward function, the learned policies may become overly conservative. Therefore, it is crucial to consider task-irrelevant dynamics uncertainty during offline model training.

To address these challenges, we propose the Separated Model-based Offline Policy Optimization (SeMOPO) method. We first analyze the performance lower bound under the Exogenous Block MDP (EX-BMDP) assumption~\cite{EXBMDP22} in the case of offline learning from noisy visual datasets. We find that such a lower bound is empirically tighter than that under the POMDP assumption~\cite{pomdp} commonly used in previous offline RL research~\cite{rafailov2021offline,2022vd4rl}. We replace the typical random sampling method in offline MBRL~\cite{yu2020mopo,rafailov2021offline,2022vd4rl} with our proposed conservative sampling method, training the separated model on trajectories collected by relatively deterministic behavior policies. After obtaining the task-relevant model from the offline dataset, we train the policy on the endogenous states imagined by this model. To evaluate our method, we construct the Low-Quality Vision Datasets for Deep Data-Driven RL (LQV-D4RL), including 15 different settings from DMControl Suite~\cite{DMC2018} and Gym~\cite{brockman2016openai} environments. SeMOPO achieves significantly better performance than all baseline methods on the LQV-D4RL. Our analytical experiments confirm the effectiveness of the conservative sampling method in identifying task-relevant information and the superiority of estimating uncertainty estimation on it during offline policy training.

The main contributions of our work are: (i) We propose a new approach named Separated Model-based Offline Policy Optimization (SeMOPO), which aims to solve offline visual RL tasks from low-quality datasets collected by sub-optimal policies and with complex distractors in observations. (ii) To establish a benchmark for the practical setting, we construct Low-Quality Vision Datasets for Deep Data-Driven RL (LQV-D4RL), which offers new research opportunities for practitioners in the field. (iii) We provide a theoretical analysis of the lower performance bound of policies learned on the endogenous state space and the superiority of our proposed conservative sampling method in differentiating task-relevant and irrelevant information. (iv) We show excellent performance of SeMOPO on LQV-D4RL, with analytical experiments validating the efficacy of each component of the method.

\section{Preliminaries}

In our work, we focus on learning skills from the offline dataset consisting of image observations, actions, and rewards, denoted as $\mathcal{B}=\{o_{1:T}^{i}, a_{1:T}^{i}, r_{1:T}^{i}\}_{i=1}^n$. The dataset contains moving distractors within the visual observations and is collected by policies $\pi_B$ with suboptimal or random behaviors. To better model the environment, we consider the Exogenous Block Markov Decision Process (EX-BMDP) setting~\cite{EXBMDP22}, an adaptation of the Block MDP~\cite{BMDP19}. A Block MDP consists of a set of observations $\mathcal{O}$; a set of latent states, $\mathcal{Z}$ with cardinality $Z$; a finite set of actions, $\mathcal{A}$ with cardinality $A$; a transition function, $T: \mathcal{Z}\times\mathcal{A}\to \Delta(\mathcal{Z})$; a reward function $R:\mathcal{O}\times\mathcal{A}\to[0,1]$; an emission function $\mathcal{U}:\mathcal{Z}\to\Delta(O)$; and an initial state distribution $\mu_0\in\Delta(\mathcal{Z})$. The agent has no access to the latent states but can only receive the observations. The block structure holds if the support of the emission distributions of any two latent states are disjoint, $\text{supp}(\mathcal{U}(\cdot|z_1))\cap \text{supp}(\mathcal{U}(\cdot|z_2))=\emptyset$ when $z_1\neq z_2$, where $\text{supp}(\mathcal{U}(\cdot|z))=\{o\in\mathcal{O}|\mathcal{U}(o|z)>0\}$, distinguishing the BMDP from the Partially Observable MDP~\cite{pomdp}. We now restate the definition of EX-BMDP:

\begin{definition}\label{definition:EX-BMDP}
    (Exogenous Block Markov Decision Process). An EX-BMDP is a BMDP such that the latent state can be decoupled into two parts $z=(s^+,s^-)$ where $s^+\in \mathcal{S}^+$ is the endogenous state and $s^-\in \mathcal{S}^-$ is the exogenous state. For $z\in\mathcal{Z}$ the initial distribution and transition functions are decoupled, that is: $\mu(z)=\mu(s^+)\mu(s^-)$, and $T(z^{\prime}|z, a)=T({s^+}^{\prime}|s^+, a)T({{s^-}}^{\prime}|s^-)$
\end{definition}
EX-BMDP decomposes the dynamics and separates the exogenous noise from the endogenous state. This noise is not controlled by the agent, but it may have a non-trivial dynamic. EX-BMDP provides a natural way to characterize this type of noise.

Model-based Offline Policy Optimization~\cite{yu2020mopo} (MOPO) is a typical offline RL method, showing the efficacy of model-based methods to learn policies from offline datasets. MOPO learns the dynamics model $\widehat{T}(\cdot|s,a)$ and the reward model $\hat{r}(s, a)$ from the fixed dataset $\mathcal{B}$, and constructs an uncertainty-penalized MDP whose reward function is defined as $\tilde{r}(s, a)=r(s, a)-u(s, a)$, where $u(s, a)$ is the model uncertainty estimation. MOPO proves that optimizing the policy in the uncertainty-penalized MDP with an admissible uncertainty estimation $u(s, a)$ is equivalent to optimizing it in the true MDP. We adopt a similar pipeline to MOPO to learn the model and policy from offline datasets.

\section{SeMOPO: Separated Model-based Offline Policy Optimization}

Model-based offline RL methods use model uncertainty to address the distribution shift problem between the online environment and the offline dataset. In~\Cref{subsec:uncertainty_estimation}, we first provide a theoretical justification of the performance bound under the Exogenous Block MDP assumption when learning from offline datasets with visual distractors and illustrate the inherent limitations of employing the POMDP framework in this scenario. In~\cref{subsec:restricted_policy}, we propose a sampling strategy to help the model differentiate task-relevant and irrelevant components and give a corresponding theoretical analysis. In~\cref{subsec:practical_imple}, we detail a practical implementation of SeMOPO, which integrates these advancements.

\subsection{Uncertainty Estimation and Performance Bound under Exogenous Block MDP}\label{subsec:uncertainty_estimation}

In this section, we theoretically analyze the lower bound of policy performance under EX-BMDP. We also provide empirical evidence that, in some specific cases, the uncertainty under EX-BMDP is lower than that under POMDP, resulting in a tighter lower performance bound. We start our analysis by extending the well-known telescoping lemma~\cite{telesopinglemma19} from the state space $\mathcal{S}$ to the endogenous state space $\mathcal{S}^+$:

\begin{lemma}\label{lemma:telescoping}
    (Telescoping lemma in the endogenous state space). Let $M$ and $\widetilde{M}$ be two MDPs with the same reward function $r$, but different dynamics $T$ and $\widetilde{T}$ respectively. Let $G_{\widetilde{M}}^\pi(s^+, a):=\underset{{s^+}^{\prime} \sim \widetilde{T}(s^+, a)}{\mathbb{E}}\left[V_M^\pi\left({s^+}^{\prime}\right)\right]-\underset{{s^+}^{\prime} \sim T(s^+, a)}{\mathbb{E}}\left[V_M^\pi\left({s^+}^{\prime}\right)\right]$. Then,
$$
\eta_{\widetilde{M}}(\pi)-\eta_M(\pi)=\gamma \underset{(s^+, a) \sim \rho_{\widetilde{T}}^\pi}{\overline{\mathbb{E}}}\left[G_{\widetilde{M}}^\pi(s^+, a)\right]
$$
\end{lemma}
Note that if $\mathcal{F}$ is a set of functions mapping $\mathcal{S}^+$ to $\mathbb{R}$ that contains $V_M^{\pi}$ for all $\pi$, then, 
\begin{equation*}
\begin{aligned}
    &|G_{\widetilde{M}}^{\pi}(s^+,a)|\\
    \le& \sup_{f\in\mathcal{F}}\Big|\mathbb{E}_{{s^+}^{\prime}\sim \widetilde{T}(s^+,a)}[f({s^+}^{\prime})]-\mathbb{E}_{{s^+}^{\prime}\sim T(s^+,a)}[f({s^+}^{\prime})]\Big|\\
    =:& d_{\mathcal{F}}(\widetilde{T}(s^+,a),T(s^+,a))
\end{aligned}
\end{equation*}
where $d_{\mathcal{F}}$ is the integral probability metric (IPM) defined by $\mathcal{F}$. The telescoping lemma provides a way to measure the performance gap between policies in the true dynamics $T$ and the estimated dynamics $\widetilde{T}$.

\begin{assumption}\label{assmp:func_class}
    Let $\mathcal{S}^+$ be the endogenous state space under the EX-BMDP. We say $\tilde{u}: \mathcal{S}^+\times \mathcal{A}\to \mathbb{R}$ is an admissible error estimator for $\widetilde{T}$ if $d_{\mathcal{F}}(\widetilde{T}(s^+,a),T(s^+,a))\le \tilde{u}(s^+, a)$ for all $s\in \mathcal{S}^+,a\in\mathcal{A}$.
\end{assumption}

\begin{theorem} \label{theorem:EXBMDP} Under~\Cref{assmp:func_class}, we can define the uncertainty estimation $\epsilon_{\tilde{u}}(\pi)$ under the EX-BMDP as $\epsilon_{\tilde{u}}(\pi):=\mathop{\bar{\mathbb{E}}}\limits_{(s^+, a)\sim\rho_{\widetilde{T}}^{\pi}}[\tilde{u}(s^+, a)]$. Let $\tilde{\pi}$ denote the learned optimal policy under the endogenous MDP, then, 
    \begin{equation*}
        \eta_{M}(\tilde{\pi})\ge \sup_{\pi}\{\eta_{M}(\pi)-2\lambda\epsilon_{\tilde{u}}(\pi)\} 
    \end{equation*}
\end{theorem}

\cref{theorem:EXBMDP} suggests that the performance under the true MDP $\mathcal{M}$ can be bounded by learning a policy in the endogenous state space using the reward penalized by the corresponding model uncertainty, which provides a theoretical guarantee on the performance of SeMOPO. The proof is in~\Cref{subapp:proof_EXBMDP}.

\textbf{Remarks.} The Partially Observable MDP used in prior works like LOMPO~\cite{rafailov2021offline} and Offline DV2~\cite{2022vd4rl} will give a biased model uncertainty estimation when learning from offline datasets containing task-irrelevant distractors in observations. These methods assume that $u: \mathcal{Z}\times \mathcal{A}\to \mathbb{R}$ can serve as an admissible error estimator for $\widehat{T}$ if $d_{\mathcal{F}}(\widehat{T}(z, a), T(z, a))\le u(z, a)$ for all $z$ and $a$, where $z$ represents the belief state of the true state $s$ within the latent space $\mathcal{Z}$. However, in addition to task-relevant information, there are task-irrelevant components in image observations that are beyond the agent's control. Without excluding them during model learning, these components will be absorbed into $z$, leading to biased uncertainty estimation based on $d_{\mathcal{F}}(\widehat{T}(z, a), T(z, a))$. Taking LOMPO as an example, the model uncertainty is estimated as the variance of the predicted probability of the next state given the current state and action, i.e., $\text{Var}\{\log \widehat{T}_{\theta_i}(z_{t+1}|z_t,a_t)\}$, which is larger than the model uncertainty $\text{Var}\{\log \widetilde{T}_{\theta_i}(s^+_{t+1}|s^+_t,a_t)\}$ under the EX-BMDP using the same calculation method. LOMPO then uses this model uncertainty as a reward penalty, leading to a loser lower performance bound than the EX-BMDP.

\begin{figure*}[tbp]
\centering
\includegraphics[width=\textwidth]{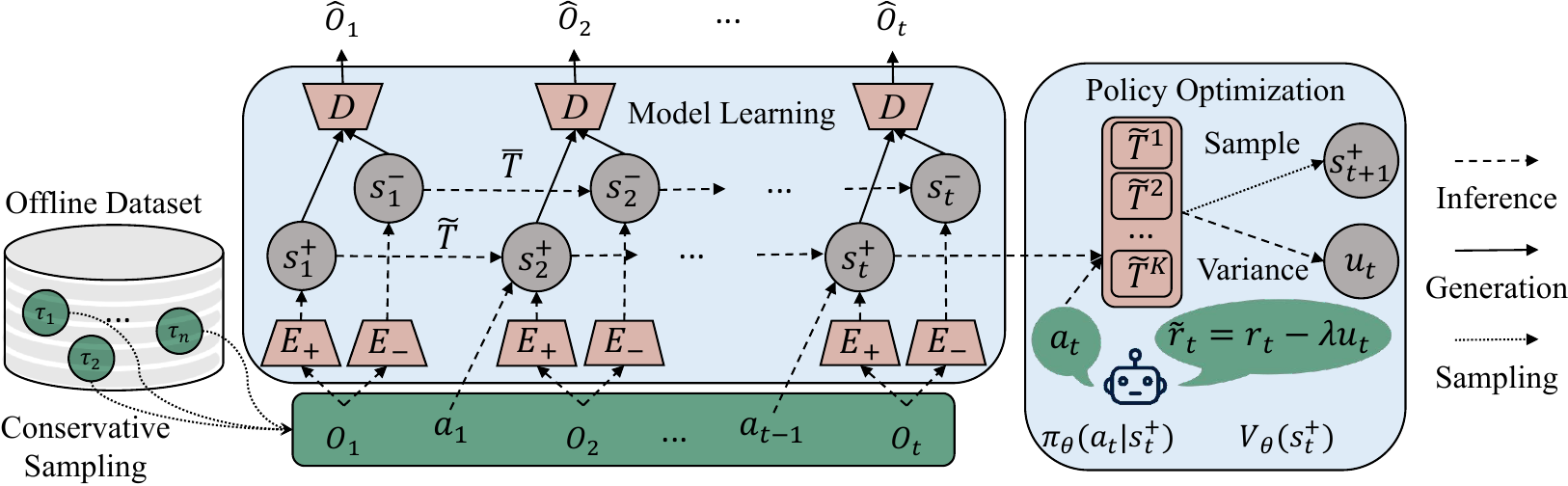}
\caption{The SeMOPO framework encompasses two parts: model learning and policy optimization. In the model learning phase, SeMOPO employs conservative sampling to select trajectories, which are then used to train models for endogenous and exogenous dynamics, each implemented as an ensemble of Gaussian distributions. During policy optimization, SeMOPO trains a policy $\pi_{\theta}(a_t|s^+_t)$ and a value model $V_{\theta}(s^+_t)$ based on the endogenous states generated by a sampled endogenous dynamics model $\tilde{T}^j$. SeMOPO uses the reward penalized by the variance of the endogenous dynamics models' predictions to train the value model.}
\label{fig:SeMOPO}
\end{figure*}

\subsection{Model Training with Conservative Sampling}\label{subsec:restricted_policy}
In~\Cref{subsec:uncertainty_estimation}, we know that a separated model under the EX-BMDP assumption benefits policy learning. In the following, we focus on how to learn such a separated model. Under the EX-BMDP assumption, the log-likelihood for the trajectory $\tau=\{o_1, a_1, \cdots, o_T, a_T\}$ can be decomposed as:
\begin{equation}
\label{eq:likelihood}
    \begin{aligned}
        \ln p(\tau) =& \sum_{t=1}^T \Big[\ln p(o_t|s^+_t,s^-_t) + \ln p(a_t|s^+_t) \\&+ \ln p(s^+_t|s^+_{t-1}, a_{t-1}) + \ln p(s^-_t|s^-_{t-1})\Big],
    \end{aligned}
\end{equation}
We omit the terms of the initial distribution of $s^+_1$ and $s^-_1$ for convenience. The detailed derivation is in~\Cref{app:deriv_likelihood}.

From the last three terms in~\Cref{eq:likelihood}, we can find that the agent's action is the key factor to help the model distinguish the task-relevant and irrelevant components. An intuitive insight is that the separated model should be updated using relatively deterministic actions. Excessively random action distributions can weaken the influence of agent actions on the endogenous transition, causing it to collapse into $p(s^+_{t+1}|s^+_t)$, making it indistinguishable from the exogenous transition. It naturally raises a question: how to sample the offline data to help decompose the endogenous and exogenous dynamics? Consider two kinds of training processes, the first is maximizing the likelihood where the data is sampled from all the mixed trajectories without replacing (\textit{Random Sampling}), and the second is maximizing the likelihood where data is sampled trajectory by trajectory (\textit{Conservative Sampling}). To formulate the relationship between these two likelihoods, we introduce~\Cref{theorem:separate}.

\begin{theorem}\label{theorem:separate}
    Consider the likelihood optimization problem on the same offline dataset $\mathcal{B}$ but with two different sampling methods. let $\mathcal{B}_{\pi_i}$ be the dataset collected by the behavior policy $\pi_i$, where $i=1,2,\cdots,n$. $\mathcal{B}_{\pi_{mix}}=\mathcal{B}_{\pi_1}\cup\mathcal{B}_{\pi_1}\cup\cdots\cup\mathcal{B}_{\pi_n}$ is the mixture of the datasets of all policies. Then we have
    \begin{align*}
          \mathbb{E}_{\tau\in \mathcal{B}_{\pi_{mix}}}\ln p(\tau) \le  \frac{1}{n}\sum_{i=1}^n \mathbb{E}_{\tau\in \mathcal{B}_{\pi_{i}}} \ln p(\tau),
    \end{align*}
    where $p(\tau)$ is the true density of $\tau$.
\end{theorem}

\Cref{theorem:separate} suggests that the total likelihood estimated on the sampled offline trajectory of a certain policy is larger than on the mixture dataset of all policies. The proof is in~\Cref{subapp:proof_{s}eparate}. Maximizing the likelihood will make the estimated distribution toward the true distribution in~\Cref{eq:likelihood}, which helps distinguish the endogenous and exogenous transitions. Notably,~\Cref{theorem:separate} does not guarantee that we can obtain the maximum likelihood via the conservative sampling method. Even so, we empirically show that the model can separate the task-relevant and irrelevant components well compared to the random sampling method.

\setlength{\tabcolsep}{2.2pt}
\begin{table*}[tbp]
\centering
\caption{Normalized test returns of SeMOPO and compared baselines on the LQV-D4RL benchmark. Mean scores (higher is better) with standard deviation are recorded across 4 seeds for each task. The original returns are shown in~\Cref{tab:original_results}.}
\begin{tabular}{@{}llccccccc@{}}
\toprule
\multicolumn{2}{c}{\textbf{Dataset}} &
  \textbf{SeMOPO} &
  \textbf{Offline DV2} &
  \textbf{LOMPO} &
  \textbf{DrQ+BC} &
  \textbf{DrQ+CQL} &
  \textbf{BC}&
  \textbf{InfoGating}\\ \midrule
\multirow{3}{*}{Walker Walk} & random  & \textbf{0.77 $\pm$ 0.06} & 0.27 $\pm$ 0.05 & 0.22 $\pm$ 0.06 & 0.04 $\pm$ 0.00 & 0.04 $\pm$ 0.00 &0.04 $\pm$ 0.00 & 0.07 $\pm$ 0.00 \\
 & medrep  & \textbf{0.87 $\pm$ 0.06} & 0.29 $\pm$ 0.04 & 0.36 $\pm$ 0.11 & 0.04 $\pm$ 0.01 & 0.03 $\pm$ 0.01 &0.04 $\pm$ 0.01 & 0.09 $\pm$ 0.03 \\
 & medium  & 0.45 $\pm$ 0.07 & 0.11 $\pm$ 0.04 & 0.10 $\pm$ 0.02 & 0.65 $\pm$ 0.06 & 0.03 $\pm$ 0.01 & \textbf{0.69 $\pm$ 0.05} & 0.16 $\pm$ 0.06 \\
\multirow{3}{*}{Cheetah Run} & random  & \textbf{0.63 $\pm$ 0.07} & 0.10 $\pm$ 0.03 & 0.16 $\pm$ 0.04 & 0.23 $\pm$ 0.08 & 0.00 $\pm$ 0.00 &0.05 $\pm$ 0.05 & 0.14 $\pm$ 0.03 \\
 & medrep  & 0.64 $\pm$ 0.07 & 0.16 $\pm$ 0.07 & 0.19 $\pm$ 0.08 & 0.41 $\pm$ 0.23 & 0.00 $\pm$ 0.00 & 0.05 $\pm$ 0.05 & \textbf{0.66 $\pm$ 0.12} \\
 & medium  & \textbf{0.73 $\pm$ 0.08} & 0.20 $\pm$ 0.14 & 0.13 $\pm$ 0.09 & 0.64 $\pm$ 0.07 & 0.00 $\pm$ 0.00 &0.62 $\pm$ 0.10 & 0.71 $\pm$ 0.09 \\
\multirow{3}{*}{Hopper Hop} & random  & 0.68 $\pm$ 0.06 & 0.00 $\pm$ 0.00 & 0.00 $\pm$ 0.00 & 0.08 $\pm$ 0.10 & 0.00 $\pm$ 0.00 &0.06 $\pm$ 0.08 & \textbf{0.79 $\pm$ 0.13} \\
 & medrep  & \textbf{0.91 $\pm$ 0.07} & 0.00 $\pm$ 0.00 & 0.00 $\pm$  0.00 & 0.25 $\pm$ 0.18 & 0.00 $\pm$ 0.00 &0.04 $\pm$ 0.02 & 0.53 $\pm$ 0.16 \\
 & medium  & \textbf{1.24 $\pm$ 0.16} & 0.02 $\pm$ 0.05 & 0.01 $\pm$ 0.04 & 0.81 $\pm$ 0.19 & 0.00 $\pm$ 0.00 &0.42 $\pm$ 0.07 & 0.58 $\pm$ 0.09 \\
\multirow{3}{*}{Humanoid Walk} & random & \textbf{0.01 $\pm$ 0.00} & 0.00 $\pm$ 0.00 & 0.00 $\pm$ 0.00 & 0.00 $\pm$ 0.00 & 0.00 $\pm$ 0.00 & 0.00 $\pm$ 0.00 & 0.00 $\pm$ 0.00 \\
& medrep & 0.01 $\pm$ 0.01 & 0.00 $\pm$ 0.00 & 0.01 $\pm$ 0.00 & 0.00 $\pm$ 0.00 & 0.00 $\pm$ 0.00 & \textbf{0.02 $\pm$ 0.01} & 0.00 $\pm$ 0.00 \\
& medium & 0.01 $\pm$ 0.01 & 0.01 $\pm$ 0.00 & 0.00 $\pm$ 0.00 & \textbf{0.02 $\pm$ 0.01} & 0.00 $\pm$ 0.00 & 0.01 $\pm$ 0.00 & 0.01 $\pm$ 0.00 \\
\multirow{3}{*}{Car Racing} & random & \textbf{0.93 $\pm$ 0.16} & 0.51 $\pm$ 0.18 & 0.86 $\pm$ 0.18 & -0.05 $\pm$ 0.05 & -0.23 $\pm$ 0.02 & -0.15 $\pm$ 0.01 & -0.20 $\pm$ 0.01 \\
& medrep & \textbf{0.80 $\pm$ 0.17} & 0.39 $\pm$ 0.18 & 0.72 $\pm$ 0.40 & -0.18 $\pm$ 0.01 & -0.23 $\pm$ 0.02 & -0.21 $\pm$ 0.02 & -0.20 $\pm$ 0.01 \\
& medium & \textbf{0.90 $\pm$ 0.34} &0.62 $\pm$ 0.32 & 0.69 $\pm$ 0.22 & 0.39 $\pm$ 0.21 & -0.21 $\pm$ 0.02 & -0.18 $\pm$ 0.00 & -0.18 $\pm$ 0.01 \\
\bottomrule
\end{tabular}\label{tab:overall_performance}
\end{table*}

\subsection{Practical Implementation of SeMOPO}\label{subsec:practical_imple}

Based on the above analysis, we present a practical implementation of \emph{Separated Model-based Offline Policy Optimization}. The overall method of SeMOPO is shown in~\Cref{fig:SeMOPO} and summarized in~\Cref{alg:SeMOPO}.

\textbf{Separated Model Learning.} \Cref{theorem:EXBMDP} allows us to achieve a significantly improved lower bound of the total return if we can learn a separate model and train the policy in the endogenous state space. By bifurcating the latent state $z$ into endogenous $(s^+)$ and exogenous $(s^-)$ parts, we obtain the following Evidence Lower Bound (ELBO):
\begin{equation*}
\begin{aligned}
    \max_{\theta}\;&
  \mathbb{E}\bigg[\sum_{t=1}^T \ln \mathcal{U}_{\theta}(o_t|s^+_t,s^-_t)\\
  &-\mathbb{D}_{KL}\big(\bar{q}_{\theta}(s^-_t|o_{\leq t})||\overline{T}_{\theta}(s^-_t|s^-_{t-1})\big)\\
  &-\mathbb{D}_{KL}\big(\tilde{q}_{\theta}(s^+_t|o_{\leq t},a_{<t})||\widetilde{T}_{\theta}(s^+_t|s^+_{t-1},a_{t-1})\big)\bigg].
\end{aligned}
\end{equation*}
where $\tilde{q}_{\theta}$ and $\bar{q}_{\theta}$ represent the inference models for the endogenous and exogenous states, respectively. Likewise, $\widetilde{T}_{\theta}$ and $\overline{T}$ denote the corresponding transition dynamics models, and $\mathcal{U}_{\theta}$ is the observation model which reconstructs the observation jointly from the endogenous and exogenous states. The derivation of the ELBO is detailed in~\cref{app:deriv}, with the implementation specifics of each model described in~\cref{app:imple}.

For the conservative sampling strategy outlined in~\cref{subsec:restricted_policy}, we design a simple but effective implementation. In the $m$-th training epoch, the SeMOPO's model is only trained on the sampled trajectory $\tau_j$ generated by a certain policy, where $j\le\min(m, n)$ and $n$ is the number of trajectories in the offline dataset. We provide an intuitive interpretation for this implementation: it forces the model to separate the task-relevant and irrelevant dynamics in the early training stages ($m\le n$); as training progresses ($m > n$), the model is trained on the trajectories sampled from the entire dataset to enhance the coverage of transitions. 

\textbf{Endogenous Model-based Offline Policy Optimization.} We train a policy $\pi_{\theta}$ and a value model $V_{\theta}$ in the endogenous state space of the learned model, following the standard training algorithm in DreamerV2~\cite{Dreamerv2}. To address the distribution shift issue, we estimate the model uncertainty through the discrepancy $d_{\mathcal{F}}(\widetilde{T}(s^+, a), T(s^+, a))$. Since we can not access the true dynamics $T$, we adopt a widely-used approach based on the model disagreements~\cite{Pathak2019SelfSupervisedEV}, which is implemented as a penalty term:
\begin{equation}\label{eq:uncertainty}
    \tilde{r}(s^+, a) = r(s^+, a)-\lambda\sum_{i=1}^K(\mu^i(s^+, a)-\bar{\mu}(s^+,a))^2
\end{equation}
Here, $\lambda$ is a coefficient for adjusting the penalty weight, and $\bar{\mu}(s^+, a)=\frac{1}{K}\sum_{i=1}^K\mu^i(s^+, a)$ is the average prediction across an ensemble of $K$ endogenous dynamics models.

\section{Experiments}\label{sec:exp}

We conduct several experiments to answer the following scientific questions: (1) Can SeMOPO outperform the existing methods with low-quality offline visual datasets? (2) Can SeMOPO give a reasonable model uncertainty estimation? (3) How does the sampling strategy for policy data affect the separated model training? (4) Can SeMOPO generalize to online environments with different distractors?

\textbf{Datasets.} 
To evaluate the efficacy of SeMOPO in the scenario of learning with low-quality offline visual datasets, we construct a dataset named LQV-D4RL (Low-Quality Vision Datasets for Deep Data-Driven RL). A ``low-quality'' dataset, in this context, refers to one where the data collection policy is either sub-optimal or randomly initialized, and the observations contain complex distractors with non-trivial dynamics. With these considerations in mind, we create nine distinct subsets for evaluation. We select four locomotion tasks — \textit{Walker Walk}, \textit{Cheetah Run}, \textit{Hopper Hop}, and \textit{Humanoid Walk} — from the DMControl Suite~\cite{DMC2018}, and the \textit{Car Racing} task from Gym~\cite{brockman2016openai}. Each task is represented across three different levels of policy performance:   
\begin{itemize}
    \item \textbf{random}: Trajectories collected by randomly initialized policies.
    \item \textbf{medium\_replay} (medrep): Trajectories drawn from the replay buffer accumulated during the training of a policy with medium performance.
    \item \textbf{medium}: Trajectories collected by a fixed policy of medium performance. 
\end{itemize}
The backgrounds of each locomotion task's observations are replaced with videos from the ``driving car'' category of the Kinetics dataset~\cite{Kinetics17}, as used in DBC~\cite{DBC21}. The size of image observation is $64\times 64\times 3$. To mimic the real data collection process in natural settings, we train policies and then collect trajectories based on image observations with the aforementioned distractors. Further details about the dataset can be found in~\cref{app:dataset}.

\textbf{Baselines.} We compare several representative methods in offline visual RL literature, dividing them into model-based and model-free categories. Model-based methods like \textbf{LOMPO}~\cite{rafailov2021offline} introduce a penalty term in the reward function based on model disagreements and optimize the policy in its constructed latent MDP, while \textbf{Offline DV2}~\cite{2022vd4rl} adapts the DreamerV2 method with a similar penalty for offline settings. Both of these two methods show great performance in offline visual RL tasks. In model-free approaches, Behavioral Cloning (BC)~\cite{BC95,Bratko1995BehaviouralCP} learns by imitating the behavior of the policy that collected the data, and Conservative Q-Learning (CQL)~\cite{Kumar2020ConservativeQF} samples actions from a broad distribution while penalizing those that fall outside the support region of the offline data. Remarkably, both BC and CQL are not originally tailored for scenarios involving high-dimensional image inputs. To address this, DrQ-v2's regularization techniques~\cite{DrQv2} are applied to BC and CQL in~\cite{2022vd4rl}, creating \textbf{DrQ+BC} and \textbf{DrQ+CQL} methods. These modified methods, alongside the original \textbf{BC}, are compared in our study to evaluate their effectiveness in offline visual RL tasks with low-quality datasets. Additionally, we compare the offline version of the \textbf{InfoGating} method~\cite{tomar2024ignorance}, a visual reinforcement learning approach that removes task-irrelevant noise by minimizing the information required for the task. We run all experiments on four seeds and report the normalized test return after training. The normalized return is obtained by normalizing the original return with the maximum and minimum values of the three levels of datasets for each task. The detailed calculation procedure is in~\Cref{app:normalized_return}.

\subsection{Evaluation on the LQV-D4RL Benchmark}\label{subsec:performance}

We evaluate SeMOPO against various baselines in nine scenarios within the LQV-D4RL benchmark. The results in~\cref{tab:overall_performance} consistently demonstrate SeMOPO's superior performance across diverse datasets, confirming the effectiveness of uncertainty estimation in the endogenous state space, especially in low-quality visual datasets. Significantly, SeMOPO outperforms model-free approaches in nearly all environments. However, BC-based approaches also perform well when the behavior policy of the dataset is reasonably effective. Model-based methods, Offline DV2 and LOMPO, show improved results when trained on random and medium\_replay datasets compared to medium datasets. This indicates that trajectories with random behaviors may provide a wider range of transitions, which benefits mitigating the distribution shift problem. This phenomenon aligns with previous research~\cite{Jin2020IsPP,Rashidinejad2021BridgingOR,offline2BC2022} and holds even in environments with complex noise in image inputs. Moreover, our findings highlight that the DrQ+BC and BC methods outperform Offline DV2 and LOMPO in several settings. This implies that inaccuracies in model uncertainty estimations can lead to worse performance of model-based methods than direct policy imitation. SeMOPO and all baseline methods have failed in the humanoid walk task. The visual reconstruction on our project website demonstrates that SeMOPO can effectively extract task information. Thus, task failure may be attributed to the inherently complex nature of the humanoid task, which involves controlling a high-dimensional action space of up to 21 dimensions. Particularly in our experiments, the challenge is further compounded by the high-dimensional image inputs and complex moving distractors in the background, significantly increasing learning difficulty. We anticipate that future researchers will focus on learning a high-performance policy for this task within noisy visual inputs. In addition, we assess these methods on the V-D4RL benchmark, as detailed in~\cref{app:overall_performance_vd4rl}, where observations do not include distractors.

\subsection{Can SeMOPO give a reasonable model uncertainty estimation?}\label{subsec:uncertainty}

\begin{figure}[tbp]
\centering
\includegraphics[width=\columnwidth]{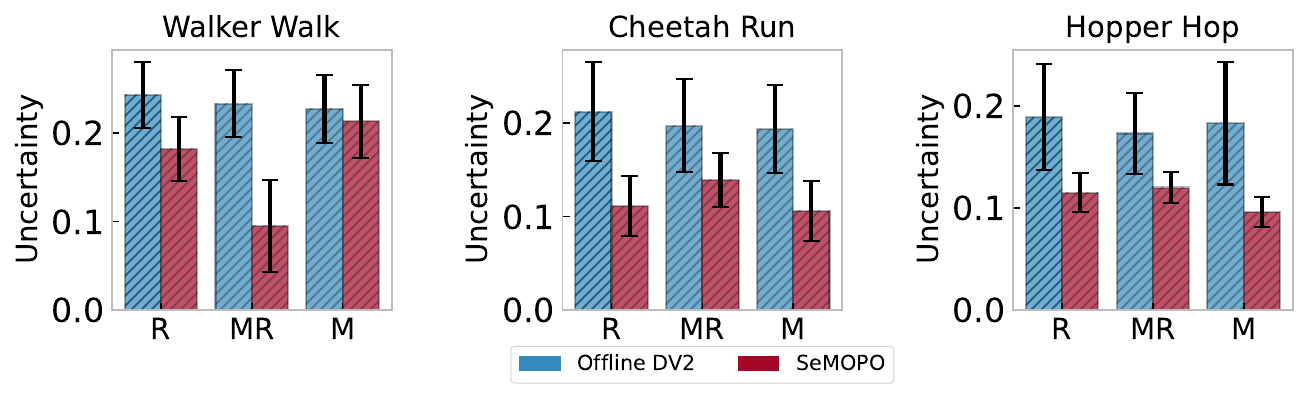}
\caption{The model uncertainty estimation of SeMOPO and Offline DV2 on the LQV-D4RL dataset. We randomly select 1000 states and report the mean and standard deviation of uncertainty on them.}
\label{fig:endog}
\end{figure}

\begin{figure}[tbp]
\centering
\includegraphics[width=\columnwidth]{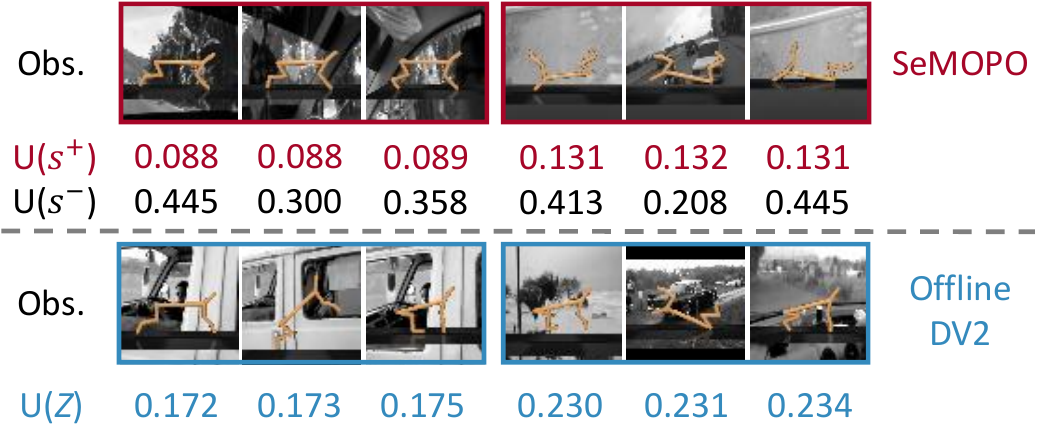}
\caption{Examples of uncertainty estimated by SeMOPO and Offline DV2. $U(s^+)$, $U(s^-)$, and $U(Z)$ represent uncertainty estimations for the endogenous state, the exogenous state, and the belief latent state, respectively.}
\label{fig:uncertainty_example}
\end{figure}

To answer the question, we compare the model uncertainty of SeMOPO and Offline DV2 on randomly selected states, as shown in~\cref{fig:endog}. We find that SeMOPO exhibits lower model uncertainty than Offline DV2 across nine datasets, confirming the conclusion of~\cref{theorem:EXBMDP}. Given that SeMOPO and Offline DV2 employ the same method for uncertainty calculation, the reduced uncertainty observed in SeMOPO can be ascribed to the endogenous state transition it introduced. We also show the uncertainty of exogenous state transitions in~\cref{app:uncertainty}.

\begin{figure}[tbp]
\vskip 0.1in
\begin{center}
\centerline{\includegraphics[width=\columnwidth]{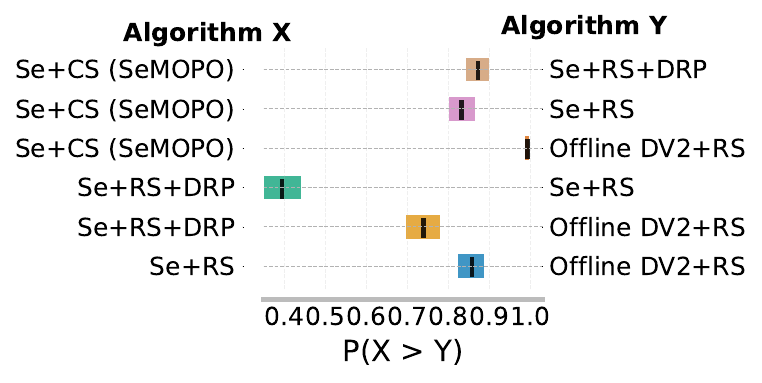}}
\caption{The performance comparison for ablated methods. Each row represents the comparative performance probabilities, complete with $95\%$ bootstrap confidence intervals, suggesting that Algorithm X is superior to Algorithm Y~\cite{Agarwal2021DeepRL}. These probabilities are derived from 50 runs of 4 seeds for every task to ensure robustness in the evaluation. We show the aggregated results for all nine tasks.}
\label{fig:ablation_pxy}
\end{center}
\vskip -0.1in
\end{figure}

To further illustrate the uncertainty estimation on unique states, we capture frames from the \textit{Cheetah Run} task and record the corresponding estimated uncertainty of SeMOPO and Offline DV2 in~\cref{fig:uncertainty_example}. SeMOPO displays lower model uncertainty on normal states as opposed to unpredictable abnormal states and shows different uncertainty of the exogenous states based on background disturbances. Offline DV2 provides similar uncertainty estimates for states with comparable background distractors but significantly different agent movements (illustrated in the three frames at the lower left corner of~\Cref{fig:uncertainty_example}). This occurs because Offline DV2 does not distinctly differentiate between endogenous and exogenous states, instead directly estimating model uncertainty on the latent states. The analysis of specific states reveals that SeMOPO can achieve reasonable model uncertainty estimation.

\begin{figure}[tbp]
\vskip 0.1in
\begin{center}
\centerline{\includegraphics[width=\columnwidth]{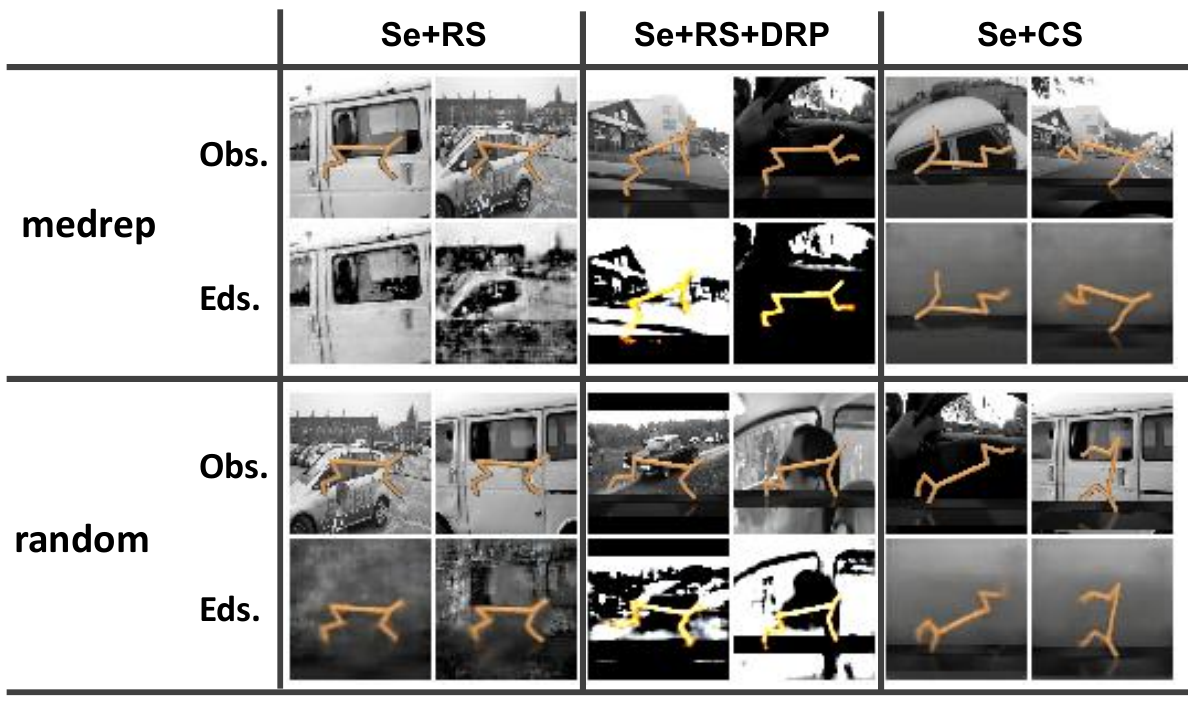}}
\caption{Original observations (Obs.) and image reconstructions from the endogenous states (Eds.) of different ablated methods in the medium\_replay and random datasets of the \textit{Cheetah Run} task. SeMOPO (Se+CS) can preserve task-relevant information well, while others can not.}
\label{fig:sampleing_res}
\end{center}
\vskip -0.1in
\end{figure}

\subsection{How does the sampling strategy for policy data affect the separated model training?}\label{subsec:sample}

To validate the efficacy of our proposed sampling method for training the separated model, we conduct a comparison involving several ablated methods: 1) Random sampling of all policy trajectories (Se + RS); 2) Random sampling combined with additional dissociated prediction loss\footnote{More information about dissociated reward prediction loss is in~\Cref{app:dissociated_reward}.} to hinder predicting the true reward from the exogenous state (Se + RS + DRP), as used in~\cite{Fu2021LearningTI}; 3) Our conservative sampling method (Se + CS (SeMOPO)). The ``Se'' indicates the application of these methods to the training of the separated model. We also compare these with the original Offline DV2 method trained under the random sampling (Offline DV2 + RS).

As illustrated in~\cref{fig:ablation_pxy}, SeMOPO demonstrates superior performances over other ablated methods, with a high probability ($>0.8$). There is no significant difference in performance between Se+RS and Se+RS+DRP ($\approx 0.4$), suggesting that the model trained by dissociated reward prediction may not benefit offline policy learning. Both methods, however, outperform the original Offline DV2, indicating the effectiveness of the separated model in addressing offline visual RL challenges with noisy observations. To further analyze these sampling strategies, we visualize the observation reconstruction of the endogenous state for each training method in~\cref{fig:sampleing_res}. The model trained with random sampling does not guarantee comprehensive task-related information in the endogenous state, especially in the medium\_replay dataset. While adding dissociated reward prediction loss to random sampling enriches task information in the endogenous state, it also introduces considerable distraction. Our conservative sampling method effectively reconstructs task-related information, ensuring the endogenous state contains only task-relevant details. 

To corroborate the theoretical claims made in~\cref{subsec:restricted_policy}, we compute the action entropy for the first 30 epochs using both RS and CS methods across all nine datasets. As shown in~\cref{fig:action_ent}, the CS method consistently results in lower action entropy compared to RS across different environments and datasets. Lower action entropy implies more deterministic policy behavior, aiding in distinguishing endogenous states from exogenous noise. Due to little action variations among different policies in the medium dataset, the difference in action entropy between RS and CS is less than the random and medium\_replay datasets. It explains why SeMOPO shows a greater advantage in the latter datasets, as described in~\Cref{tab:overall_performance}.

\begin{figure}[tbp]
\vskip 0.1in
\begin{center}
\centerline{\includegraphics[width=\columnwidth]{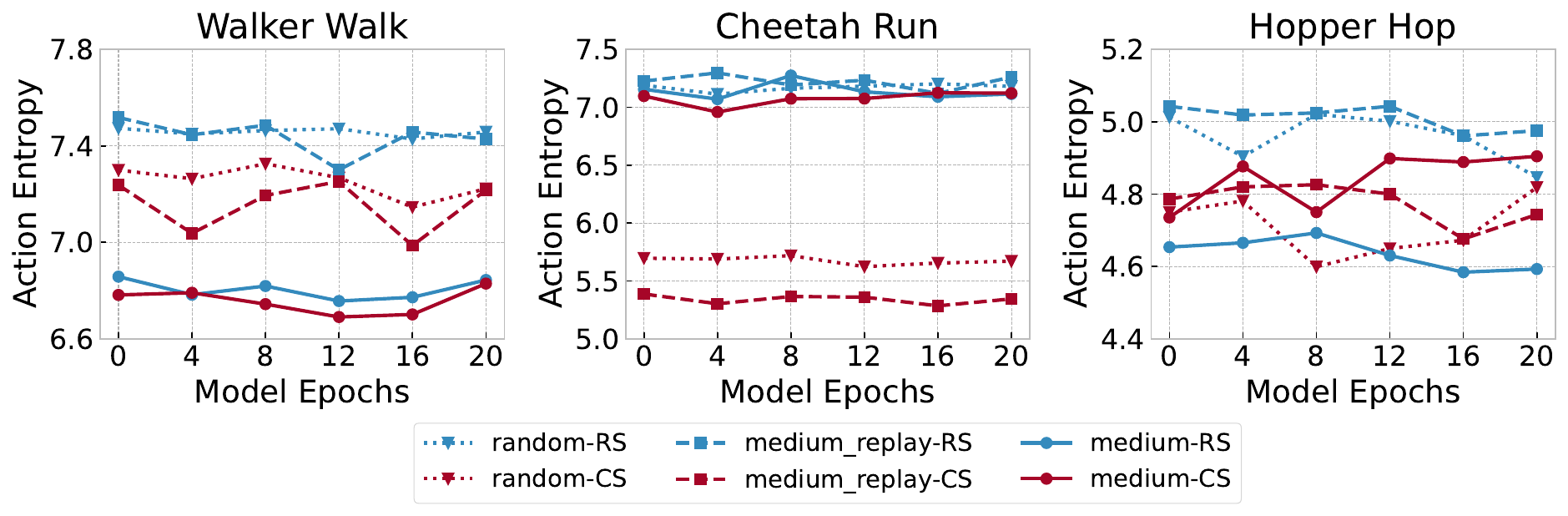}}
\caption{Action entropy of sampled data by random sampling (RS) and conservative sampling (CS). We record values of the first 30 training epochs across nine datasets of LQV-D4RL. The conservative sampling yields data with significantly lower action entropy than random sampling.}
\label{fig:action_ent}
\end{center}
\vskip -0.1in
\end{figure}

\subsection{Can SeMOPO generalize to online environments with different distractors?}\label{subsec:generalize}

In this section, we evaluate the effectiveness of SeMOPO in addressing the ``offline-to-online distraction gap" challenge. We alter the online testing environment by introducing an RGB background, starkly contrasting with the grayscale background of the LQV-D4RL dataset. To assess the methods' adaptability to distinct online visual distractions, we compared SeMOPO with Offline DV2 and DrQ+BC in two distinct environments: \textit{Walker Walk} and \textit{Cheetah Run}, each with two datasets - ``random" and ``medrep". As shown in~\Cref{tab:transfer}, Offline DV2 performs better than DrQ+BC in the \textit{Walker Walk} task, but the reverse is true for the \textit{Cheetah Run} task. However, both are outperformed by SeMOPO, demonstrating SeMOPO's clear advantage in managing the distraction gap. 

\setlength\tabcolsep{0.8pt}
\begin{table}[tbp]
\centering
% \begin{small}
     \caption{Normalized test returns on tasks with different online distractions from offline datasets.}
	\begin{tabular}{lcccc}
		\toprule
		\multicolumn{2}{c}{\textbf{Dataset}}     & \multicolumn{1}{c}{\textbf{SeMOPO}} & \multicolumn{1}{c}{\textbf{Offline DV2}} & \multicolumn{1}{c}{\textbf{DrQ+BC}} \\ \midrule
		\multirow{2}{*}{Walker Walk} & random         & \textbf{0.78 $\pm$ 0.14} & 0.33 $\pm$ 0.05 &  0.05 $\pm$ 0.01 \\
		                             & medrep          & \textbf{0.50 $\pm$ 0.02} & 0.30 $\pm$ 0.05 & 0.05 $\pm$ 0.02 \\
		\multirow{2}{*}{Cheetah Run} & random         & \textbf{0.65 $\pm$ 0.05}  & 0.16 $\pm$ 0.04 & 0.45 $\pm$ 0.05 \\
		                             & medrep          &  \textbf{0.59 $\pm$ 0.06} & 0.17 $\pm$ 0.06 & 0.49 $\pm$ 0.14  \\ \bottomrule
	\end{tabular}
% \end{small}
\label{tab:transfer}
\end{table}

\section{Related Work}
\label{sec:relatedwork}

\textbf{Offline RL.} Offline RL~\cite{lange2012batch,Levine2020OfflineRL} aims to learn a policy from an offline dataset of trajectories for deployment in an online environment. It has been applied to various domains, such as robotic grasping~\cite{2018QTOptSD,2021MTOptCM}, healthcare~\cite{Shortreed2010InformingSC,Wang2018SupervisedRL}, and autonomous driving~\cite{offlinerl_autodriving2021,offlinerl_autodriving2023}. The primary challenge in offline RL is the distribution shift~\cite{Fujimoto2018OffPolicyDR,Kumar2019StabilizingOQ}, where the behavioral policy's data diverges from the data distribution in the actual online environment. Methods to address distribution shift broadly fall into two categories: (1) \textit{policy-constraints} methods that restrict the learned policy to align closely with the behavior policy generating the dataset~\cite{Fujimoto2018OffPolicyDR,Kumar2019StabilizingOQ,Liu2019OffPolicyPG,Nachum2019AlgaeDICEPG,Peng2019AdvantageWeightedRS,Siegel2020KeepDW,Fujimoto2021AMA}, thus avoiding unexpected actions; (2) \textit{conservative} methods~\cite{Kumar2020ConservativeQF,morel20,Kostrikov2021OfflineRL,yu2020mopo} that construct a conservative return estimate, enhancing the learned policy's robustness against distribution shift. Previous studies have shown offline RL's efficacy even with random or sub-optimal datasets~\cite{Jin2020IsPP,Zanette2020ExponentialLB,Wang2020WhatAT,Rashidinejad2021BridgingOR}. However, there is limited research on the performance of offline RL in scenarios where the observation data is noisy, particularly when learning from high-dimensional image inputs with complex distractors. Our work focuses on learning a policy from such low-quality offline visual RL datasets to effectively tackle related tasks.

\textbf{Control with Noisy Observations.} Real-world control tasks often involve observations that contain irrelevant noise. Recent studies that focused on learning policies from noisy image observations, can be categorized into four types: (1) separating task-relevant and irrelevant information based on key factors (actions, rewards, etc.)~\cite{Fu2021LearningTI,Wang2022DenoisedML,Pan2022IsolatingAL,Liu2023LearningWM}; (2) learning task-related representations through bisimulation metrics~\cite{DBC21,Liu2023RobustRL}; (3) mitigating exogenous noises via data augmentation methods~\cite{Kostrikov2020DrQ,hansen2021stabilizing,Fan2021DRIBO,Yuan2022DontTW,Bertoin2022LookWY,huang2022spectrum}; (4) extracting task-related features using auxiliary prediction tasks~\cite{yang2015auxiliary,Badia2020Agent57OT,Baker2022VideoP,Efroni2022SampleEfficientRL,Lamb2022GuaranteedDO}. These methods primarily target training policies in online environments, where additional data can be gathered to differentiate between task-relevant and irrelevant information. However, this issue has been less explored in offline visual RL. The V-D4RL benchmark dataset~\cite{2022vd4rl} involves only two related subtasks and does not offer a specific approach for handling noisy observations. Our method distinctively trains a separated state-space model from offline noisy data using a conservative sampling method and learns a high-performance policy.

\textbf{Sampling Strategy in Offline RL.}
Sampling strategies in offline RL aim to improve the performance of the learning agent by optimally selecting data from the dataset. One approach involves using uncertainty estimation of the Q-Value function to guide sampling~\cite{Kumar2022OfflineRR}, which allows for a better exploration-exploitation balance. Another strategy is Offline Prioritized Experience Replay (OPER)~\cite{Hong2023BeyondUS}, which assigns weights to transitions based on their normalized advantage, prioritizing those with higher rewards. Rank-Based Sampling (RBS)~\cite{Shen2021ARS} adopts a similar way, sampling transitions with high returns more frequently. Policy Regularization with Dataset Constraint (PRDC)~\cite{Ran2023PRDC} limits the policy towards the closest state-action pair in the dataset, avoiding out-of-distribution actions. To address the imbalance dataset issue, RB-CQL~\cite{Jiang2023OfflineRL} utilizes a retrieval process to use related experiences effectively. Unlike these methods, which directly target  policy performance improvement, our proposed conservative sampling strategy focuses on helping the model distinguish between task-relevant and irrelevant information.

\section{Conclusion}
\label{sec:conclusion}

In this paper, we consider the problem of learning a policy from offline visual datasets where the observations contain non-trivial distractors and the behavior policies that generate the dataset are either random or sub-optimal. To simulate this scenario, we construct the LQV-D4RL benchmark. We provide a theoretical analysis of the lower performance bound under the EX-BMDP assumption which is tighter than that under the POMDP in some specific cases. We propose a conservative sampling method to facilitate the decomposition of endogenous and exogenous states. Guided by this theoretical framework, we introduce the Separated Model-based Offline Policy Optimization (SeMOPO) method. Experimental results on the LQV-D4RL dataset indicate that SeMOPO outperforms other offline visual RL methods. Further experiments validate the theory's applicability and highlight the significance of each component within our method. Generalization experiments show SeMOPO's ability to handle variations between online and offline environmental disturbances. 

\textbf{Limitations and Future Work.} Our study adopts the independence assumption under the EX-BMDP for the transitions of endogenous and exogenous states. However, potential interactions between these states in real-world tasks may present a promising avenue for future research.

\section*{Impact Statement}

This paper presents work that aims to advance the field of Machine Learning. There are many potential societal consequences of our work, none of which we feel must be specifically highlighted here.

\section*{Acknowledgements}

We thank Han-Jia Ye, Shaowei Zhang, Minghao Shao, Hai-Hang Sun, Kaichen Huang, Yucen Wang, and Jiayi Wu for their valuable discussions. This work was partially supported by the National Science and Technology Major Project under Grant No. 2022ZD0114805, Collaborative Innovation
Center of Novel Software Technology and Industrialization, NSFC (62376118,
62006112, 62250069, 61921006), the Postgraduate Research \& Practice Innovation Program of Jiangsu Province, and the National Natural Science Foundation of China (No. 123B2009).


\bibliography{example_paper}
\bibliographystyle{icml2024}

%%%%%%%%%%%%%%%%%%%%%%%%%%%%%%%%%%%%%%%%%%%%%%%%%%%%%%%%%%%%%%%%%%%%%%%%%%%%%%%
%%%%%%%%%%%%%%%%%%%%%%%%%%%%%%%%%%%%%%%%%%%%%%%%%%%%%%%%%%%%%%%%%%%%%%%%%%%%%%%
% APPENDIX
%%%%%%%%%%%%%%%%%%%%%%%%%%%%%%%%%%%%%%%%%%%%%%%%%%%%%%%%%%%%%%%%%%%%%%%%%%%%%%%
%%%%%%%%%%%%%%%%%%%%%%%%%%%%%%%%%%%%%%%%%%%%%%%%%%%%%%%%%%%%%%%%%%%%%%%%%%%%%%%
\newpage
\appendix
\onecolumn

\section{Proof}\label{app:proof}

\subsection{Proof of Theorem \ref{theorem:EXBMDP}}\label{subapp:proof_EXBMDP}

\begin{lemma}
    (\Cref{lemma:telescoping}) (Telescoping lemma in the endogenous state space). Let $M$ and $\widetilde{M}$ be two MDPs with the same reward function $r$, but different dynamics $T$ and $\widetilde{T}$ respectively. Let $G_{\widetilde{M}}^\pi(s^+, a):=\underset{{s^+}^{\prime} \sim \widetilde{T}(s^+, a)}{\mathbb{E}}\left[V_M^\pi\left({s^+}^{\prime}\right)\right]-\underset{{s^+}^{\prime} \sim T(s^+, a)}{\mathbb{E}}\left[V_M^\pi\left({s^+}^{\prime}\right)\right]$. Then,
$$
\eta_{\widetilde{M}}(\pi)-\eta_M(\pi)=\gamma \underset{(s^+, a) \sim \rho_{\widetilde{T}}^\pi}{\overline{\mathbb{E}}}\left[G_{\widetilde{M}}^\pi(s^+, a)\right]
$$
\end{lemma}

\begin{proof}

We adopt the same proof as~\cite{telesopinglemma19,yu2020mopo} to prove the telescoping lemma in the endogenous state space. Let $W_j$ be the expected return when executing $\pi$ on $\widetilde{T}$ for the first $j$ steps, then switching to $T$ for the remainder. That is,
$$
W_j=\underset{\substack{t<j: s^+_{t+1} \sim \widetilde{T}\left(s^+_t, a_t\right) \\ t \geq j: s^+_{t+1} \sim T\left(s^+_t, a_t\right)}}{\mathbb{E}}\left[\sum_{t=0}^{\infty} \gamma^t r\left(s^+_t, a_t\right)\right]
$$

Note that $W_0=\eta_M(\pi)$ and $W_{\infty}=\eta_{\widetilde{M}}(\pi)$, so
$$
\eta_{\widetilde{M}}(\pi)-\eta_M(\pi)=\sum_{j=0}^{\infty}\left(W_{j+1}-W_j\right)
$$

Write
$$
\begin{aligned}
W_j & =R_j+\underset{s^+_j, a_j \sim \pi, \widetilde{T}}{\mathbb{E}}\left[\underset{s^+_{j+1} \sim T\left(s^+_t, a_t\right)}{\mathbb{E}}\left[\gamma^{j+1} V_M^\pi\left(s^+_{j+1}\right)\right]\right] \\
W_{j+1} & =R_j+\underset{s^+_j, a_j \sim \pi, \widetilde{T}}{\mathbb{E}}\left[\underset{s^+_{j+1} \sim \widetilde{T}\left(s^+_t, a_t\right)}{\mathbb{E}}\left[\gamma^{j+1} V_M^\pi\left(s^+_{j+1}\right)\right]\right]
\end{aligned}
$$
where $R_j$ is the expected return of the first $j$ time steps, which are taken with respect to $\widetilde{T}$. Then
$$
\begin{aligned}
W_{j+1}-W_j & =\gamma^{j+1} \underset{s^+_j, a_j \sim \pi, \widetilde{T}}{\mathbb{E}}\left[\underset{{s^+}^{\prime} \sim \widetilde{T}\left(s^+_j, a_j\right)}{\mathbb{E}}\left[V_M^\pi\left({s^+}^{\prime}\right)\right]-\underset{{s^+}^{\prime} \sim T\left(s^+_j, a_j\right)}{\mathbb{E}}\left[V_M^\pi\left({s^+}^{\prime}\right)\right]\right] \\
& =\gamma^{j+1} \underset{s^+_j, a_j \sim \pi, \widetilde{T}}{\mathbb{E}}\left[G_{\widetilde{M}}^\pi\left(s^+_j, a_j\right)\right]
\end{aligned}
$$

Thus
$$
\begin{aligned}
\eta_{\widetilde{M}}(\pi)-\eta_M(\pi) & =\sum_{j=0}^{\infty}\left(W_{j+1}-W_j\right) \\
& =\sum_{j=0}^{\infty} \gamma^{j+1} \underset{s^+_j, a_j \sim \pi, \widetilde{T}}{\mathbb{E}}\left[G_{\widetilde{M}}^\pi\left(s^+_j, a_j\right)\right] \\
& =\gamma_{(s^+, a) \sim \rho_{\widetilde{T}}^\pi}^{\mathbb{E}}\left[G_{\widetilde{M}}^\pi(s^+, a)\right]
\end{aligned}
$$

\end{proof}

\begin{theorem} (\Cref{theorem:EXBMDP}) Under~\Cref{assmp:func_class}, we can define the uncertainty estimation $\epsilon_{\tilde{u}}(\pi)$ under the EX-BMDP as $\epsilon_{\tilde{u}}(\pi):=\mathop{\bar{\mathbb{E}}}\limits_{(s^+, a)\sim\rho_{\tilde{T}}^{\pi}}[\tilde{u}(s^+, a)]$. Let $\tilde{\pi}$ denote the learned optimal policy under reward-penalized the endogenous MDP, then, 
    \begin{equation*}
        \eta_{M}(\tilde{\pi})\ge \sup_{\pi}\{\eta_{M}(\pi)-2\lambda\epsilon_{\tilde{u}}(\pi)\} 
    \end{equation*}
\end{theorem}

\begin{proof}
With~\Cref{lemma:telescoping}, we can get the corollary that:
\begin{equation*}
    \eta_{M}(\pi)=\bar{\mathbb{E}}_{(s^+,a)\sim\rho_{\widetilde{T}}^{\pi}}\big[r(s^+,a)-\gamma G_{\widetilde{M}}^{\pi}(s^+,a)\big]\ge \bar{\mathbb{E}}_{(s^+,a)\sim\rho_{\widetilde{T}}^{\pi}}\big[r(s^+,a)-\gamma|G^{\pi}_{\widetilde{M}}(s^+,a)|\big]
\end{equation*}

With~\Cref{assmp:func_class}, we know that $|G_{\widetilde{M}}^{\pi}(s^+,a)|\le \tilde{u}(s^+,a)$. Thus, we define the penalized reward function as $\tilde{r}(s^+,a)=r(s^+,a) - \lambda u(s^+,a)$, where $\lambda=\gamma c$ and $c$ is the scalar that satisfies $|G_{\widetilde{M}^{\pi}}(s^+,a)|\le c d_{\mathcal{F}}(\widetilde{T}(s^+,a),T(s^+,a))$. We can obtain the relationship between the policy performance under the true MDP and the reward-penalized endogenous MDP $\widetilde{M}_{\tilde{r}}=(\mathcal{S},\mathcal{A},\widetilde{T},\tilde{r},\mu_0, \gamma)$:
\begin{equation}\label{eq:penalized_end_mdp}
\begin{aligned}
    \eta_{M}(\pi)&\ge \bar{\mathbb{E}}_{(s^+,a)\sim\rho_{\widetilde{T}}^{\pi}}\big[r(s^+,a)-\gamma|G^{\pi}_{\widetilde{M}}(s^+,a)|\big]\\
    &\ge \bar{\mathbb{E}}_{(s^+,a)\sim \rho_{\widetilde{T}}^{\pi}}\big[r(s^+,a)-\lambda \tilde{u}(s^+, a)\big]  \\
    &\ge \bar{\mathbb{E}}_{(s^+,a)\sim \rho_{\widetilde{T}}^{\pi}}\big[\tilde{r}(s^+,a)\big]\\
    &=\eta_{\widetilde{M}_{\tilde{r}}}(\pi)
\end{aligned}
\end{equation}

From~\Cref{assmp:func_class} We can easily get the two-sided bound that:
\begin{equation}\label{eq:two_side_bound}
    |\eta_{M}(\pi) - \eta_{\widetilde{M}}(\pi)|\le \bar{\mathbb{E}}_{(s^+,a)\sim\rho_{\widetilde{T}}^{\pi}}|\gamma G_{\widetilde{M}}^{\pi}(s^+,a)|\le \lambda\bar{\mathbb{E}}_{(s^+,a)\sim\rho_{\widetilde{T}}^{\pi}}[\tilde{u}(s^+,a)] =\lambda \epsilon_{\tilde{u}}(\pi)
\end{equation}

With the help of~\Cref{eq:penalized_end_mdp} and~\Cref{eq:two_side_bound}, we can get the performance lower bound of the learned optimal policy $\tilde{\pi}$ in the reward-penalized endogenous MDP:

\begin{align*}
    \eta_{M}(\tilde{\pi})&\ge \eta_{\widetilde{M}_{\tilde{r}}}(\tilde{\pi})\ge \eta_{\widetilde{M}_{\tilde{r}}}(\pi)= \eta_{\widetilde{M}}(\pi) - \lambda\epsilon_{\tilde{u}}(\pi)\ge \eta_{M}(\pi) - 2\lambda\epsilon_{\tilde{u}}(\pi)\\\Rightarrow \quad 
    \eta_{M}(\tilde{\pi})&\ge \sup_{\pi}\{\eta_{M}(\pi)-2\lambda\epsilon_{\tilde{u}}(\pi)\}
\end{align*}
\end{proof}

\subsection{Proof of Theorem \ref{theorem:separate}}\label{subapp:proof_{s}eparate}

\begin{assumption}\label{assump:equal}
    There exists an underlying function $f: \mathcal{S}^+\times \mathcal{S}^- \rightarrow \mathcal{Z}$ satisfying that for any $s^+_1,s^+_2 \in \mathcal{S}^+, s^-_1,s^-_2 \in \mathcal{S}^-$, if $f(s^+_1,s^-_1) = f(s^+_2,s^-_2)$, then
    \begin{align}
        p(o_t|s^+_1,s^-_1)= p(o_t|s^+_2,s^-_2)
    \end{align}
\end{assumption}

This assumption is intuitive because the block structure indicates that each context $o$ uniquely determines its generating state $z$. However, $z$ can have multiple different partitions of $s^{+}$ and $s^{-}$ based on specific semantics, such as task-related information, viewpoint, etc.

\begin{assumption} \label{assump:seperability}
    For each $i$, the distribution of action for trajectory $\tau_i$ is the same, i.e., there exist $p(a)$ such that
    \begin{align*}
     \mathbb{E}_{s^{+}}\pi_i(a|s^{+}) = p(a),\quad i=1,2,\cdots,n.
    \end{align*}
\end{assumption}

\begin{theorem}
    (\Cref{theorem:separate}) Consider the likelihood optimization problem on the same offline dataset $\mathcal{B}$ but with two different sampling methods. let $\mathcal{B}_{\pi_i}$ be the dataset collected by the behavior policy $\pi_i$, where $i=1,2,\cdots,n$. $\mathcal{B}_{\pi_{mix}}=\mathcal{B}_{\pi_1}\cup\mathcal{B}_{\pi_1}\cup\cdots\cup\mathcal{B}_{\pi_n}$ is the mixture of the datasets of all policies. Then we have
    \begin{align*}
          \mathbb{E}_{\tau\in \mathcal{B}_{\pi_{mix}}}\ln p(\tau) \le  \frac{1}{n}\sum_{i=1}^n \mathbb{E}_{\tau\in \mathcal{B}_{\pi_{i}}} \ln p(\tau),
    \end{align*}
    where $p(\tau)$ is the true density of $\tau$.
\end{theorem}

\begin{proof}
For that $\tau$ generated under the guide of $\pi$, the action distribution of the policy, we maximize the log-likelihood to evaluate the conditional distributions $p(o|s^+,s^-),p(a|s^+),p({s^+}^{\prime}|s^+,a)$ and $p({s^-}^{\prime}|s^-)$ which also denoted as $p_o,p_a,p_{s^+},p_{s^-}$ respectively for convenience, i.e., 
\begin{align*}
     \hat{p}_o,\hat{p}_a,\hat{p}_{s^+},\hat{p}_{s^-}= \underset{p_o,p_a,p_{s^+},p_{s^-} \in \mathcal{H}_\Theta}{\arg \max }&\mathbb{E}_{\tau \in \mathcal{B}_\pi} \ln p(\tau)\\
     =\underset{p_o,p_a,p_{s^+},p_{s^-} \in \mathcal{H}_\Theta}{\arg \max} &\mathbb{E}_{\tau \in \mathcal{B}_\pi} \sum_{t=1}^T \Big(\ln p_o(o_t|s^+_t,s^-_t) + \ln p_a(a_t|s^+_t) \\
     &+ \ln p_{s^+}(s^+_t|s^+_{t-1}, a_{t-1}) + \ln p_{s^-}(s^-_t|s^-_{t-1})\Big).
\end{align*}
Here $\mathcal{H}_\Theta$ is the approximate space and $\Theta$ is the class of parameters. In practice, $\mathcal{H}_{\Theta}$ is the variational posterior space with an encoder and decoder for each conditional distribution.

With Assumption \ref{assump:equal}, we know that once the distribution of $f(s^+_t,s^-_t)$ is fixed, the variant of $s^+_t$ and $s^-_t$ has no effect on the observation $o_t$. Therefore, if we want to separate sub-optimal distributions for $s^+_t$ and $s^-_t$ from the optimal results, the key terms in the likelihood are those involving the action $a_t$. Formally, we treat the likelihood sequentially as
\begin{align*}
    \mathbb{E}_{\tau \in \mathcal{B}_\pi} \ln p(\tau):= \sum_{t=1}^T l^{(t)}_{\pi}(s_t^+,s_t^-,a) + \mathbb{E}_{\tau\in \mathcal{B}_\pi} \ln p_o(o_t|s^+_t,s^-_t),
\end{align*}
where 
\begin{align*}
    l^{(t)}_{\pi}(s_t^+,s_t^-,a)  = \mathbb{E}_{a\sim \pi} \ln p_a(a_t|s_t^+) + \ln p_{s_t^+}(s_t^+|s^+_{t-1}, a_{t-1}) + \ln p_{s_t^-}(s_t^-|s^-_{t-1}).
\end{align*}
Define the equivalent class for $s^+,s^-$ as 
\begin{align*}
    \mathcal{A}_{p} = \left\{s^+\sim p_{s^+},s^-\sim p_{s^-}: f(s^+,s^-)\sim p\right\}.
\end{align*}
Different elements in $\mathcal{A}_p$ means different distributions for $s^+$ and $s^-$ that induce the same distribution for $f(s^+,s^-)$. From Assumption \ref{assump:equal}, we know that for those $s^+,s^-$ in a certain class $\mathcal{A}_{p_f}$, $\mathbb{E}_{\tau\in \mathcal{B}_\pi} \ln p_o(o_t|s^+_t,s^-_t)$ is only related to the distribution $p_f$ and we denote its value as $l_o(p_f)$. For $t = 1$, we know
\begin{align*}
    l^{(1)}_{\pi}(s^+,s^-,a) =\mathbb{E}_{\pi} \ln p_a(a|s^+) + \ln p_{s^+}(s^+) + \ln p_{s^-}(s^-).
\end{align*}
Recall the conditional entropy and mutual information entropy for $a$ and $s^+$ such that
\begin{align*}
    H(a) = \int - \ln p_a(a) d p_a(a),\quad H(a|s^+) = \int d p_{s^+}(s^+)\int - \ln p_a(a|s^+) dp_a(a|s^+),\quad I(a,s^+) = H(a) - H(a|s^+).
\end{align*}

Under Assumption \ref{assump:seperability}, consider two different type of policy $\pi_{i}$ as the policy of $i$-th curve and $\pi_{mix}$ as the mixture policy of all curves, i.e., $\pi_{mix} = \frac{1}{n}\sum_{i=1}^n \pi_{i}$. We know that
\begin{align*}
    I(a\sim \pi_{mix},s^+\sim p_{s^+}) &= H(a) + \int d p_{s^+}(s^+)\int  \frac{1}{n}\sum_{i=1}^n \pi_{i}(a|s^+)\ln \frac{1}{n}\sum_{i=1}^n \pi_{i}(a|s^+) d a \\& \le 
    H(a) + \int d p_{s^+}(s^+)\int \frac{1}{n}\sum_{i=1}^n \pi_i(a|s^+) \ln \pi_{i}(a|s^+) da\\&=
    \frac{1}{n} \sum_{i=1}^n I(a\sim \pi_i,s^+\sim p_{s^+}).
\end{align*}
Define $p^{\star}_{s^+}(s^+_{t=1}),p^{\star}_{s^-}(s^-_{t=1})\in \mathcal{A}_{P_{t=1}^{\star}}$ as the true distribution of $s^+$ and $s^-$ at time $t=1$. Once the class $\mathcal{A}_{p_f}$ and the distribution of $s^+$ is fixed, the distribution for $s^-$ is also fixed. So we know that
\begin{align*}
    &l_{\pi_{mix}}^{(1)}(s^+\sim p^{\star}_{s^+}(s^+_{t=1}),s^-\sim p^{\star}_{s^-}(s^-_{t=1}),a) - \frac{1}{n}\sum_{i=1}^n l_{\pi_{i}}^{(1)}(s^+\sim p^{\star}_{s^+}(s^+_{t=1}),s^-\sim p_{s^-}^{\star}(s^-_{t=1}),a)\\
    = \ & I(a\sim \pi_{mix},s^+\sim p^{\star}_{s^+})  + \mathbb{E}_{\pi_{mix}}\ln p^{\star}_{s^+}(s^+)  +\ln p^{\star}_{s^-}(s^-) \\
    &- I(a\sim \pi_{mix},s^+\sim p^{\star}_{s^+})  - \frac{1}{n}\sum_{i=1}^n \left(\mathbb{E}_{\pi_{i}}\ln p^{\star}_{s^+}(s^+)  +\ln p^{\star}_{s^-}(s^-)\right)\\
    = \ &I(a\sim \pi_{mix},s^+\sim p^{\star}_{s^+}) -
      \frac{1}{n}\sum_{i=1}^n I(a\sim \pi_{i},s^+\sim p^{\star}_{s^+}) \le 0.
\end{align*}
For $t>1$, similarly define that  $p^{\star}_{s^+}(s^+_t|s^+_{t-1},a_{t-1}),p^{\star}_{s^-}(s^-_t|s^-_{t-1})\in \mathcal{A}_{P_{t>1}^{\star}}$ as the true condition distribution of $s^+_t$ condition on $s^+_{t-1}$ and $s^-_t$ condition on $s^-_{t-1}$. So
\begin{align*}
    \mathbb{E}_{a\sim \pi_{mix}} \ln p^{\star}_{s^+}(s^+_t|s^+_{t-1},a_{t-1}) &= \int \pi_{mix}(a_{t-1}|s^+_{t-1})\ln p^{\star}_{s^+}(s^+_t|s^+_{t-1},a_{t-1}) d a_{t-1} \\&=\frac{1}{n} \mathbb{E}_{a\sim \pi_{i}} \ln p^{\star}_{s^+}(s^+_t|s^+_{t-1},a_{t-1}).
\end{align*}
Then we also have
\begin{align*}
    &l_{\pi_{mix}}^{(t)}(s^+\sim p^{\star}_{s^+}(s^+_t|s^+_{t-1},a_{t-1}),s^-\sim p^{\star}_{s^-}(s^-_t|s^-_{t-1}),a) - \frac{1}{n}\sum_{i=1}^n l_{\pi_{i}}^{(t)}(s^+\sim p^{\star}_{s}(s^+_t|s^+_{t-1},a_{t-1}),s^-\sim p_{s^-}^{\star}(s^-_t|s^-_{t-1}),a)\\= \ & I(a\sim \pi_{mix},s^+\sim \prod_{k=1}^{t}\mathbb{E}_{s^+_{k-1},a_{t-1}}p^{\star}_{s^+}(s^+_k|s^+_{k-1},a_{t-1}))  + \mathbb{E}_{\pi_{mix}}\ln p^{\star}_{s^+}(s^+_t|s^+_{t-1},a_{t-1})  +\ln p^{\star}_{s^-}(s^-_t|s^-_{t-1}) \\-\ &
      \frac{1}{n}\sum_{i=1}^n I(a\sim \pi_{i},s^+\sim \prod_{k=1}^{t}\mathbb{E}_{s^+_{k-1},a_{t-1}}p^{\star}_{s^+}(s^+_k|s^+_{k-1},a_{t-1}))  - \frac{1}{n}\sum_{i=1}^n \left(\mathbb{E}_{\pi_{i}}\ln p^{\star}_{s^+}(s^+_t|s^+_{t-1},a_{t-1})  +\ln p^{\star}_{s^-}(s^-_t|s^-_{t-1})\right)\\= \ &I(a\sim \pi_{mix},s^+\sim \prod_{k=1}^{t}\mathbb{E}_{s^+_{k-1},a_{t-1}}p^{\star}_{s^+}(s^+_k|s^+_{k-1},a_{t-1})) -
      \frac{1}{n}\sum_{i=1}^n I(a\sim \pi_{i},s^+\sim \prod_{k=1}^{t}\mathbb{E}_{s^+_{k-1},a_{t-1}}p^{\star}_{s^+}(s^+_k|s^+_{k-1},a_{t-1})) \le 0.
\end{align*}
Then
\begin{align*}
    \mathbb{E}_{\tau\in \mathcal{B}_{\pi_{mix}}}\ln p^{\star}(\tau)&=  \sum_{t=1}^T l_{\pi_{mix}}^{(t)}(s^+\sim p^{\star}_{s^+}(s^+_t|s^+_{t-1}),s^-\sim p^{\star}_{s^-}(s^-_t|s^-_{t-1}),a)  +l_o(p_t^{\star}) \\& \le 
    \frac{1}{n}\sum_{i=1}^n \sum_{t=1}^T l_{\pi_{i}}^{(t)}(s^+\sim p^{\star}_{s^+}(s^+_t|s^+_{t-1}),s^-\sim p^{\star}_{s^-}(s^-_t|s^-_{t-1}),a)  +l_o(p_t^{\star})\\& = \frac{1}{n}\sum_{i=1}^n \mathbb{E}_{\tau\in \mathcal{B}_{\pi_{i}}} \ln p^{\star}(\tau).
\end{align*}
Besides, if we can select a subset $\mathcal{C}$ of curves $i$ with large mutual information entropy $I(a\sim\pi_i|s^+\sim p^{\star}_{s^+})$, i.e., $\frac{1}{\#\mathcal{C}}\sum_{i \in \mathcal{C}} I(a\sim\pi_i|s^+\sim p^{\star}_{s^+}) > I(a\sim\pi_{mix}|s^+\sim p^{\star}_{s^+})$ and we know that
\begin{align*}
     \mathbb{E}_{\tau\in \mathcal{B}_{\pi_{mix}}}\ln p^{\star}(\tau) < \frac{1}{\#\mathcal{C}}\sum_{i\in \mathcal{C}}\mathbb{E}_{\tau\in \mathcal{B}_{\pi_{i}}}\ln p^{\star}(\tau).
\end{align*}

\end{proof}

\section{Derivation}\label{app:deriv}

\paragraph{Likelihood.}\label{app:deriv_likelihood}
For the trajectory $\tau_i=\{o_1,a_1,\cdots,o_T,a_T\}$, the log-likelihood can be expressed as:
\begin{align*}
    \ln p(\tau_i) =& \ln p(o_0, a_0, \cdots, o_T, a_T)\\
    =& \ln\prod_{t=1}^{T}p(o_t|z_t)p(a_t|z_t)p(z_{t}|z_{t-1}, a_{t-1})\\
    =&  \ln\prod_{t=1}^{T}p(o_t|s^+_t, s^-_t)p(a_t|s^+_t)p(s^+_{t}|s^+_{t-1}, a_{t-1})p(s^-_t|s^-_{t-1})\\
    =&  \sum_{t=1}^T\big[\ln p(o_t|s^+_t,s^-_t) + \ln p(a_t|s^+_t) + \ln p(s^+_{t}|s^+_{t-1}, a_{t-1}) + \ln p(s^-_{t}|s^-_{t-1})\big],
\end{align*}
The third equation comes from the decomposition of endogenous and exogenous dynamics, as assumed in the EX-BMDP framework.

\paragraph{One-step predictive distribution.}\label{app:deriv_elbo}

The variational bound for latent dynamics models $p(o_{1:T},z_{1:T}|a_{1:T})=\prod_t p(s^+_t|s^+_{t-1},a_{t-1})p(s^-_t|s^-_{t-1})p(o_t|z_t)$ and a variational posterior $q(z_{1:T}|o_{1:T},a_{1:T})=\prod_t q(s^+_t|o_{\leq t},a_{<t})q(s^-_t|o_{\leq t})$ follows from importance weighting and Jensen's inequality as shown:

\begin{align*}
\ln p(o_{1:T}|a_{1:T})\triangleq&\ln\mathbb{E}_{p(z_{1:T}|a_{1:T})}\Big[\prod_{t=1}^T p(o_t|z_t)\Big] \\
=&\ln\mathbb{E}_{q(z_{1:T}|o_{1:T},a_{1:T})}\Big[\prod_{t=1}^T \frac{p(o_t|z_t)p(z_t|z_{t-1},a_{t-1})}{q(z_t|o_{\leq t},a_{<t})}\Big] \\
=&\ln\mathbb{E}_{q(s^+_{1:T}, s^-_{1:T}|o_{1:T},a_{1:T})}\Big[\prod_{t=1}^T \frac{p(o_t|s^+_t,s^-_t)p(s^+_t|s^+_{t-1},a_{t-1})p(s^-_t|s^-_{t-1})}{q(s^+_t|o_{\leq t},a_{<t})q(s^-_t|o_{\leq t})}\Big] \\
\geq&\mathbb{E}_{q(s^+_{1:T}, s^-_{1:T}|o_{1:T},a_{1:T})}\Big[\sum_{t=1}^T \big(\ln p(o_t|s^+_t, s^-_t)+\ln p(s^+_t|s^+_{t-1},a_{t-1})+\ln p(s^-_t|s^-_{t-1})\nonumber\\
&-\ln q(s^+_t|o_{\leq t},a_{<t})-\ln q(s^-_t|s^-_{t-1})\big)\Big] \\
=&\sum_{t=1}^T \Big(
  \mathbb{E}_{q(s^+_t|o_{\leq t},a_{<t})q(s^-_t|o_{\leq t})}\bigg[\ln p(o_t|s^+_t,s^-_t)\bigg]\nonumber\\
  &-\mathbb{E}_{q(s^+_{t-1}|o_{\leq t-1},a_{<t-1})}\bigg[\text{KL}[q(s^+_t|o_{\leq t},a_{<t})||p(s^+_t|s^+_{t-1},a_{t-1})]\bigg]\nonumber\\
  &-\mathbb{E}_{q(s^-_{t-1}|o_{\leq t-1})}\bigg[\text{KL}[q(s^-_t|o_{\leq t})||p(s^-_t|s^-_{t-1})]\bigg] \Big).
\end{align*}

\section{The LQV-D4RL Benchmark}\label{app:dataset}

To evaluate the performance of visual RL methods with offline datasets containing noisy observations, we introduce a benchmark named Low-Quality Vision Datasets for Deep Data-Driven RL (LQV-D4RL). This benchmark comprises four typical environments from the DeepMind Control Suite and one environment from Gym:
\begin{itemize}
\item \textit{Walker Walk}: A bipedal agent is trained to first stand and then walk forward as efficiently as possible.
\item \textit{Cheetah Run}: A cheetah-like bipedal model aims to run at high speeds on a straight track.
\item \textit{Hopper Hop}: The agent, with a single-legged body, must balance and hop forward, focusing on agility and stability.
\item \textit{Humanoid Walk}: A simplified humanoid with 21 joints, aims to walk stably, which is extremely difficult with many local minima.
\item \textit{Car Racing}: A highly challenging racing game where players must pass through checkpoints to score points. The faster they reach the finish line within a set time, the higher their score. The observations contain numerous distractors. 
\end{itemize}
Each task is represented across three different levels of policy performance:
\begin{itemize}
\item \textbf{Random}: Trajectories are collected by randomly initialized policies.
\item \textbf{Medium Replay} (medrep): Trajectories are drawn from the replay buffer accumulated during training of a medium-performance policy.
\item \textbf{Medium}: Trajectories are collected by a fixed policy of medium performance.
\end{itemize}
For each locomotion task's observations, the backgrounds are replaced with videos from the ``driving car" category of the Kinetics dataset~\cite{Kinetics17}, as utilized in DBC~\cite{DBC21}. To simulate real data collection processes in natural settings, we train policies using the TIA approach~\cite{Fu2021LearningTI} and then collect trajectories based on image observations with the aforementioned distractors. The ``random" and ``medium" datasets each contain 200 trajectories, while ``medium\_replay" comprises 400 trajectories, with each trajectory being 1000 steps long. Specific statistical details of the LQV-D4RL benchmark are reported in~\Cref{tab:d4rlp-stats}. We upload the dataset in the supplementary materials.

\setlength{\tabcolsep}{7.0pt}
\begin{table}[ht]
\centering
\caption{Full summary statistics of per-episode return in the \textsc{LQV-D4RL} benchmark.}
\label{tab:d4rlp-stats}
\begin{tabular}{@{}llcccccccc@{}}
\toprule
\multicolumn{2}{c}{\textbf{Dataset}} &
  \textbf{Episodes} &
  \textbf{Mean} &
  \textbf{Std} &
  \textbf{Min} &
  \textbf{P25} &
  \textbf{Median} &
  \textbf{P75} &
  \textbf{Max} \\ \midrule
\multirow{3}{*}{Walker Walk}      & random & 200 & 86.6 & 48.1 & 5.9 & 51.9 & 70.8 & 128.3 & 199.1  \\
                             & medrep  & 400 & 106.6 & 79.7 & 5.0 & 45.4 & 81.3 & 148.8 & 398.2 \\
                             & medium & 200 & 513.1 & 54.8 & 401.5 & 471.4 & 516.0 & 559.1 & 598.0 \\
                             % & medexp & 400 & 783.8 & 188.9 & 401.2 & 572.4 & 871.2 & 935.5 & 983.0 \\
                             % & expert & 200 & 897.3 & 49.1 & 801.0 & 865.4 & 900.0 & 937.6 & 983.5 \\ \midrule
\multirow{3}{*}{Cheetah Run}     & random & 200 & 77.6 & 57.2 & 3.1 & 14.5 & 71.6 & 118.8 & 198.2  \\
                             & medrep  & 400 &145.4 & 113.3 & 1.7 & 42.9 & 125.1 & 232.6 & 396.5 \\
                             & medium & 200 & 350.3 & 30.8 & 300.5 & 322.2 & 352.5 & 376.0 & 403.2 \\
                             % & medexp & 400 & 503.2 & 146.1 & 300.8 & 354.2 & 602.8 & 629.2 & 794.5 \\
                             % & expert & 200 & 636.1 & 34.3 & 600.5 & 612.5 & 626.5 & 647.6 & 816.5 \\ 
\multirow{3}{*}{Hopper Hop}     & random & 200 & 2.1 & 4.9 & 0.0 & 0.0 & 0.0 & 0.1 & 19.5  \\
                             & medrep  & 400& 4.7 & 9.7 & 0.0 & 0.0 & 0.0 & 3.7 & 39.9 \\
                             & medium & 200 & 62.0 & 13.9 & 40.3 & 49.8 & 60.8 & 74.5 & 84.9\\
\multirow{3}{*}{Humanoid Walk}     & random & 200  & 1.1 & 0.8 & 0.0 & 0.5 & 1.0 & 1.5 & 5.7  \\
                             & medrep  & 400 & 95.6 & 114.3 & 0.0 & 1.4 & 5.4 & 202.8 & 359.0 \\
                             & medium & 200  & 573.0 & 16.8 & 526.5 & 560.6 & 572.9 & 584.9 & 609.4 \\
\multirow{3}{*}{Car Racing}     & random & 200 & 10.3 & 65.5 & -82.0 & -43.3 & -6.5 & 59.8 & 149.2 \\
                             & medrep  & 400 & 76.1 & 116.3 & -82.0 & -27.5 & 54.1 & 181.5 & 297.3 \\
                             & medium & 200  & 372.3 & 42.7 & 302.3 & 335.5 & 370.3 & 408.3 & 449.9 \\
\bottomrule
\end{tabular}
\end{table}

% \begin{figure}[ht]
% \centering
% \includegraphics[width=\textwidth]{figs/lqvd4rl_all.pdf}
% \caption{The distribution of episodic returns of cheetah run and walker walk in LQV-D4RL dataset.}
% \end{figure}

\section{Implementation Details}\label{app:imple}

\subsection{Networks}\label{app:networks}

We implement the proposed algorithm using TensorFlow 2 and conduct all experiments on an NVIDIA RTX 3090, totaling approximately 1000 GPU hours. The recurrent state-space model from DreamerV2~\cite{Dreamerv2} is employed for both forward dynamics and the posterior encoder. The hidden sizes of deterministic and stochastic parts of the model are 200 and 32, respectively. For learning an ensemble of forward dynamics, multiple MLP networks are utilized, each outputting the mean and standard deviation of the next state. The hidden size for each MLP is 1024. The reward predictor comprises 4 MLP layers, each of size 400. We use the convolutional encoder and decoder from TIA \cite{Fu2021LearningTI}. All dense layers have a size of 400, and the activation function used is ELU. The ADAM optimizer is employed to train the network with batches of 64 sequences, each of length 50. The learning rate is 6e-5 for both the endogenous and exogenous models and 8e-5 for the action and value nets. We stabilize the training process by clipping gradient norms to 100 and set ($\lambda=10$) for the uncertainty penalty term. The imagine horizon of 5, as used in Offline DV2~\cite{2022vd4rl}, is adopted for policy optimization. We train a model comprising the dynamics and reward predictor for 25000 epochs using offline visual datasets, followed by policy training within the model for 100000 steps. The codes and datasets are contained in the supplementary materials.

\subsection{Evaluation Metric}\label{app:normalized_return}

To compare the performance of different methods, we define the normalized return based on the statistics of the offline dataset as follows:
\begin{equation*}
    S_{\text{normalize}} = \frac{S_{\text{score}} - S_{\text{min}}}{S_{\text{max}} - S_{\text{min}}}
\end{equation*}
Here, $S_{\text{score}}$ represents the original return obtained by the test method, while $S_{\text{min}}$ and $S_{\text{max}}$ denote the minimum and maximum episodic returns of the dataset for each task across three levels (random, medium\_replay, medium), respectively.

\subsection{Dissociated Reward Prediction}\label{app:dissociated_reward}

The Reward Dissociation used in TIA~\cite{Fu2021LearningTI} for the exogenous model is achieved through the adversarial objective $\mathcal{J}^t_{Radv}$.
\begin{equation*}
    \mathcal{J}_{Radv}^t=-\lambda_{Radv}\max_q\ln q(r_t|s_t^-)
\end{equation*} 
where $\lambda_{Radv}$ for \textit{Walker Walk}, \textit{Cheetah Run}, and \textit{Hopper Hop} are 20000, 20000, and 30000, respectively. This setup involves a minimax strategy, wherein the training of the exogenous model's reward prediction head is interleaved with the exogenous model’s training, occurring for multiple iterations per training step. The reward prediction head is trained to minimize the reward prediction loss, represented by $-\ln q(r_t|s^-_t)$. In contrast, the exogenous model aims to maximize this objective to prevent reward-correlated information from influencing its learned features, as outlined by~\citeauthor{UDA15}. At the same time, TIA optimizes the endogenous model to maximize the log-likelihood of predicting rewards from endogenous states via the objective $\mathcal{J}_R^t=\ln q(r_t|s_t^+)$. The reward prediction loss is calculated using $\ln \mathcal{N} (r_t; \hat{r}_t, 1)$, where $\mathcal{N}(\cdot; \mu, \sigma^2)$ denotes the Gaussian likelihood and $\hat{r}_t$ represents the predicted reward. Notably, neither loss is used to update the endogenous and exogenous models in SeMOPO. The reward prediction loss only updates the reward predictors by stopping the backward gradients to the endogenous states.

\section{Algorithm of SeMOPO}\label{sec:algo}

The pseudo-code of our proposed SeMOPO is provided in~\cref{alg:SeMOPO}.
\begin{algorithm}
   \caption{Training Procedure of SeMOPO}
   \label{alg:SeMOPO}
\begin{algorithmic}
   \STATE {\bfseries Input:} Offline datasets $\mathcal{B}$\\
    \STATE Initialize forward dynamics model $\widetilde{T}_{\theta}, \bar{T}_{\theta}$, posterior encoder $\tilde{q}_{\theta}, \bar{q}_{\theta}$, observation decoder $\widetilde{U}_{\theta}$, reward predictor $R_{\theta}$, policy $\pi_{\theta}$, value model $V_{\theta}$.
    \STATE {// Offline model training}
    \FOR{each training epoch $m=1\cdots M$}
        \STATE {Sample minibatch $(o_{1:T}, a_{1:T-1}, r_{1:{T}})_{1:b}$ from the dataset $\mathcal{B}$ via Conservative Sampling}
        \STATE {Update the forward dynamics model $\widetilde{T}_{\theta},\bar{T}_{\theta}$ and the posterior encoder $\tilde{q}_{\theta}, \bar{q}_{\theta}$}
        \STATE {Update the observation decoder $\mathcal{U}_{\theta}$ and the reward predictor $R_{\theta}$}
    \ENDFOR
    \STATE {// Policy Optimization}
    \FOR{training iteration $i=1\cdots \text{It}$}
        \STATE {Imagine the endogenous latent states $s^+_{1:H}$ by policy $\pi$ using the endogenous state model $\widetilde{T}_{\theta}$}
        \STATE {Estimate the endogenous model uncertainty by the model disagreement}
        \STATE {Obtain the penalized reward $\tilde{r}$ with estimated uncertainty $\tilde{u}$ via~\Cref{eq:uncertainty}}
        \STATE {Train the policy $\pi_{\theta}$ and value model $V_{\theta}$ on the data $(s^+_{1:H-1}, a_{1:H-1}, \tilde{r}_{1:H-1})$}
    \ENDFOR
\end{algorithmic}
\end{algorithm}

\section{Additional Results}

\subsection{The Model Uncertainty Estimation}\label{app:uncertainty}

To validate that the accuracy of model uncertainty estimation is due to the separation of endogenous and exogenous states, rather than a specific computational approach, we employ two different uncertainty estimation methods: the mean-disagreement of the ensemble (md) from Offline DV2, and the variance of logarithm prediction (vlp) from LOMPO. The task-related model uncertainty across various environments and datasets is shown in~\Cref{fig:end_all}. The model uncertainty of SeMOPO is superior under both computational methods compared to other techniques, suggesting that learning the model in the endogenous state space can reduce the estimated uncertainty. 

\Cref{fig:exo_all} presents the estimates of model uncertainty in both endogenous and exogenous state spaces. Notably, the model uncertainty in the exogenous state space is significantly higher than in the endogenous state space. This further implies that the overestimation of model uncertainty in task-related components in previous work is attributable to the unfiltered noise in the latent states.

\begin{figure}[tbp]
\vskip 0.1in
\begin{center}
\centerline{\includegraphics[width=\textwidth]{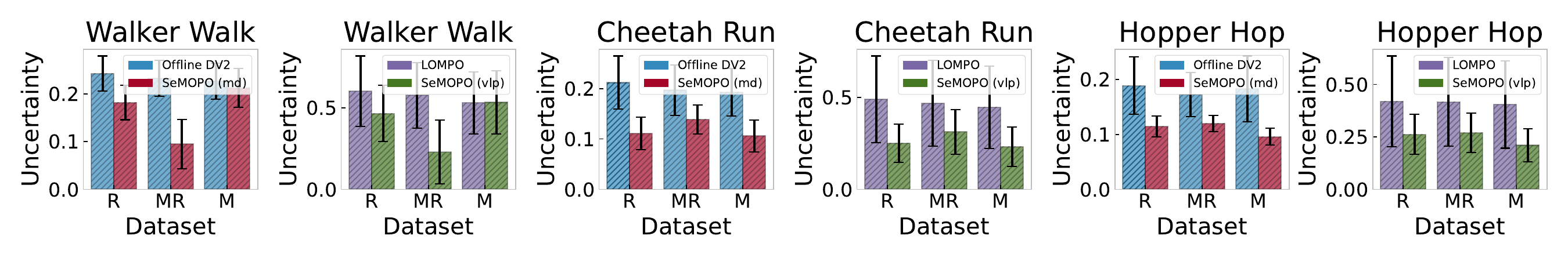}}
\caption{The model uncertainty estimation of SeMOPO, Offline DV2, and LOMPO on the LQV-D4RL dataset. We randomly select 1000 states and report the mean and standard deviation of uncertainty on them. ``md" and ``vlp" are the short names for the mean-disagreement and variance of logarithm prediction, denoting the calculation of uncertainty used in Offline DV2 and LOMPO, respectively.}
\label{fig:end_all}
\end{center}
\vskip -0.1in
\end{figure}

\begin{figure}[tbp]
\vskip 0.1in
\begin{center}
\centerline{\includegraphics[width=\textwidth]{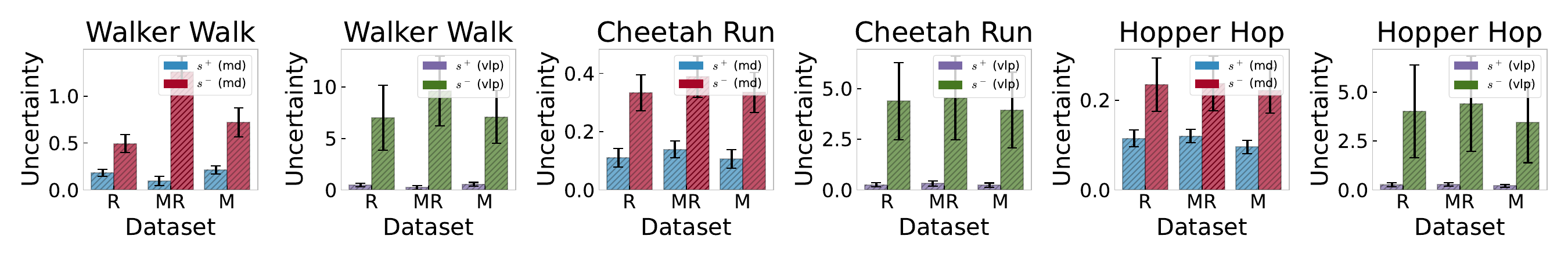}}
\caption{The model uncertainty estimation of SeMOPO on the LQV-D4RL dataset. We randomly select 1000 states and infer the endogenous states $s^+$ and exogenous states $s^-$ by SeMOPO. We report the mean and standard deviation of uncertainty on these inferred states. ``md" and ``vlp" follow the same definition as in~\Cref{fig:end_all}.}
\label{fig:exo_all}
\end{center}
\vskip -0.1in
\end{figure}

\subsection{Evaluation on the V-D4RL Benchmark}\label{app:overall_performance_vd4rl}

To assess the performance of our method on datasets without distractors, we compared SeMOPO with other approaches on the V-D4RL dataset. The results in~\Cref{tab:overall_performance_vd4rl} show that SeMOPO outperforms other methods on the random datasets of two environments, and achieves comparable performance on several other datasets. This indicates that our method is more suitable for datasets collected using non-expert policies, and can also address certain offline visual reinforcement learning problems without distractors.

\setlength{\tabcolsep}{7.0pt}
\begin{table*}[tbp]
\centering
\caption{The performance of different methods on the V-D4RL benchmark. We report the mean and standard deviation of test returns of SeMOPO over 4 seeds. The results for methods other than SeMOPO are sourced from Table 1 in~\cite{2022vd4rl}.}
\label{tab:overall_performance_vd4rl}
\begin{tabular}{@{}llcccccc@{}}
\toprule
\multicolumn{2}{c}{\textbf{Dataset}} &
  \textbf{SeMOPO} &
  \textbf{Offline DV2} &
  \textbf{LOMPO} &
  \textbf{DrQ+BC} &
  \textbf{DrQ+CQL} &
  \textbf{BC}\\ \midrule
\multirow{3}{*}{Walker Walk}      
& random & \textbf{305 $\pm$ 8} & 287 $\pm$ 130 &  219 $\pm$ 81 & 55 $\pm$ 9 & 144 $\pm$ 124 & 20 $\pm$ 2\\
                             & medrep  & 218 $\pm$ 21 & \textbf{565 $\pm$ 181} & 347 $\pm$ 197 & 287 $\pm$ 69 &114 $\pm$ 124 & 165 $\pm$ 43 \\
                             & medium & 406 $\pm$ 30 & 434 $\pm$ 111 & 341 $\pm$ 197 &   \textbf{468 $\pm$ 23} & 148 $\pm$ 161 &409 $\pm$ 31 \\
                             % & medexp &  & 78 $\pm$ 17 &  &  366 $\pm$ 157 & 25 $\pm$ 3 &558 $\pm$ 104 \\
                             % & expert &  & 20 $\pm$ 04 &  &   287 $\pm$ 189 & 29 $\pm$ 04 &479 $\pm$ 34\\ \midrule
\multirow{3}{*}{Cheetah Run}    
& random  & \textbf{330 $\pm$ 2} & 329 $\pm$ 2 & 114 $\pm$ 51  &  58 $\pm$ 6 & 59 $\pm$ 84 &0 $\pm$ 0  \\
                             
                             & medrep  & 410 $\pm$ 13 & \textbf{616 $\pm$ 10} & 363 $\pm$ 136 &  448 $\pm$ 36 & 107 $\pm$ 128 & 250 $\pm$ 36 \\
                             & medium & 190 $\pm$ 22 & 172 $\pm$ 35 & 164 $\pm$ 83 &  \textbf{530 $\pm$ 30} & 409 $\pm$ 51 &516 $\pm$ 14  \\
                             % & medexp &  & 68 $\pm$ 33 &  &  252 $\pm$ 107 & 00 $\pm$ 00 &243 $\pm$ 69  \\
                             % & expert &  & 30 $\pm$ 06 &  &  126 $\pm$ 70 & 00 $\pm$ 00 &258 $\pm$ 70\\ 
\bottomrule
\end{tabular}
\end{table*}

\subsection{Training Stability}

The gradient norms during model training for Offline DV2 and SeMOPO are shown in~\Cref{fig:grad_norm}.

\begin{figure}[tbp]
\vskip 0.1in
\begin{center}
\centerline{\includegraphics[width=\textwidth]{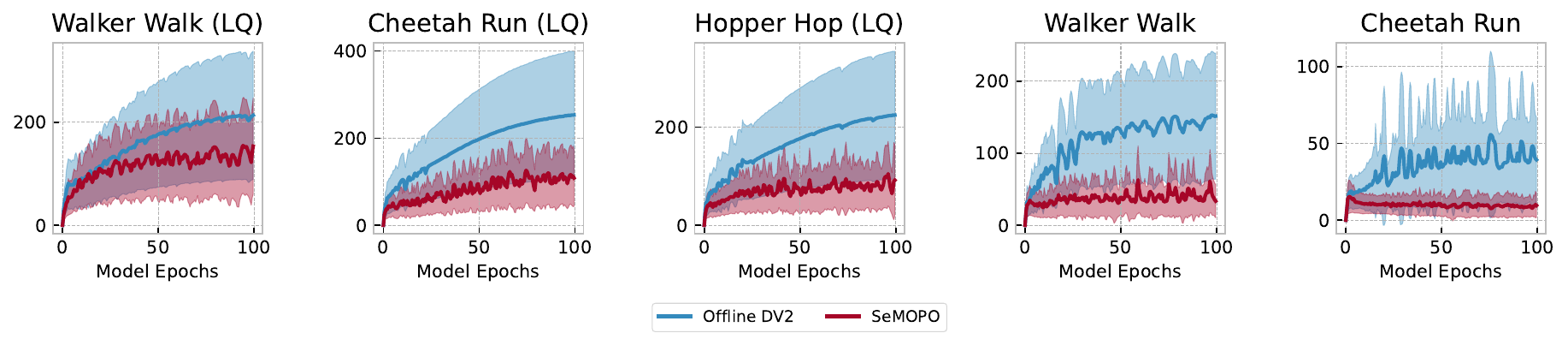}}
\caption{Comparison of gradient norms during model training for Offline DV2 and SeMOPO across five tasks: the first three from LQV-D4RL and the latter two from V-D4RL. Each curve represents aggregated data from three levels of datasets (random, medium\_replay, medium), illustrating the mean (solid line) and standard deviation (shaded region) over four seeds. SeMOPO exhibits lower model gradient norms than Offline DV2, regardless of distractors in observations, indicating that the separation of task-relevant information from observations contributes to more stable model training.}
\label{fig:grad_norm}
\end{center}
\vskip -0.1in
\end{figure}

\subsection{The results on the LQV-D4RL benchmark}

We record the original unnormalized returns for each method in~\Cref{tab:original_results}.

\setlength{\tabcolsep}{2.8pt}
\begin{table*}[tbp]
\centering
\caption{The unnormalized test returns of different methods on the LQV-D4RL benchmark. Mean scores (higher is better) with standard deviation are recorded across 4 seeds for each task.}
\label{tab:original_results}
\begin{tabular}{@{}llccccccc@{}}
\toprule
\multicolumn{2}{c}{\textbf{Dataset}} &
  \textbf{SeMOPO} &
  \textbf{Offline DV2} &
  \textbf{LOMPO} &
  \textbf{DrQ+BC} &
  \textbf{DrQ+CQL} &
  \textbf{BC}&
  \textbf{InfoGating}\\ \midrule
\multirow{3}{*}{Walker Walk}      & random  & \textbf{459 $\pm$ 41} & 167 $\pm$ 36 & 133 $\pm$ 40 &  27 $\pm$ 1 & 26 $\pm$ 2 &25 $\pm$ 2 & 49 $\pm$ 7 \\
                             & medrep   & \textbf{521 $\pm$ 42} & 175 $\pm$ 31 & 218 $\pm$ 69 & 27 $\pm$ 3 & 25 $\pm$ 3 &28 $\pm$ 7 & 58 $\pm$ 22\\
                             & medium & 273 $\pm$ 47 & 68 $\pm$ 26 & 63 $\pm$ 15 &  390 $\pm$ 36 & 25 $\pm$ 3 & \textbf{413 $\pm$ 30} & 98 $\pm$ 41\\
                             % & medexp &  & 7.8 $\pm$ 1.7 &  &  36.6 $\pm$ 15.7 & 2.5 $\pm$ 3 &55.8 $\pm$ 10.4 \\
                             % & expert &  & 2.0 $\pm$ 0.4 &  &   28.7 $\pm$ 18.9 & 2.9 $\pm$ 0.4 &47.9 $\pm$ 3.4\\ \midrule
\multirow{3}{*}{Cheetah Run}     & random  & \textbf{254 $\pm$ 30} & 40 $\pm$ 15 & 66 $\pm$ 18 &  94 $\pm$ 30 & 0 $\pm$ 0 &22 $\pm$ 20 & 58 $\pm$ 16 \\
                             & medrep   & 258 $\pm$ 28 & 66 $\pm$ 31 & 78 $\pm$ 32 &  166 $\pm$ 93 & 1 $\pm$ 0 &20 $\pm$ 18 & \textbf{267 $\pm$ 52} \\
                             & medium & \textbf{293 $\pm$ 32} & 81 $\pm$ 57 & 52 $\pm$ 37 &  260 $\pm$ 28 & 0 $\pm$ 0 &252 $\pm$ 40 & 288 $\pm$ 39 \\
                             % & medexp &  & 6.8 $\pm$ 3.3 &  &  25.2 $\pm$ 10.7 & 0.0 $\pm$ 0.0 &24.3 $\pm$ 6.9  \\
                             % & expert &  & 3.0 $\pm$ 0.6 &  &  12.6 $\pm$ 7.0 & 0.0 $\pm$ 0.0 &25.8 $\pm$ 7.0\\ 
\multirow{3}{*}{Hopper Hop}     & random & 58 $\pm$ 5 &  0 $\pm$ 0  & 0 $\pm$ 0  & 7 $\pm$ 8 & 0 $\pm$ 0 &5 $\pm$ 6 & \textbf{67 $\pm$ 11} \\
                             & medrep  & \textbf{77 $\pm$ 6} &  0 $\pm$ 0 & 0 $\pm$ 0  &   21 $\pm$ 15 & 0 $\pm$ 0 &3 $\pm$ 1 & 45 $\pm$ 14 \\
                             & medium & \textbf{105 $\pm$ 14} &  2 $\pm$ 4 & 1 $\pm$ 3 &  68 $\pm$ 16 & 0 $\pm$ 0 &35 $\pm$ 5 & 49 $\pm$ 8 \\
\multirow{3}{*}{Humanoid Walk} & random & \textbf{6 $\pm$ 3} & 3 $\pm$ 1 & 2 $\pm$ 2 & 1 $\pm$ 1 & 3 $\pm$ 2 & 1 $\pm$ 1 & 3 $\pm$ 2 \\
                               & medrep & 7 $\pm$ 4 & 2 $\pm$ 1 & 5 $\pm$ 2 & 2 $\pm$ 1 & 2 $\pm$ 2 & \textbf{10 $\pm$ 7} & 2 $\pm$ 2 \\
                               & medium & 6 $\pm$ 5 & 4 $\pm$ 2 & 3 $\pm$ 2 & \textbf{14 $\pm$ 7} & 2 $\pm$ 1 & 4 $\pm$ 2 & 4 $\pm$ 3 \\
\multirow{3}{*}{Car Racing} & random & \textbf{418 $\pm$ 79} & 233 $\pm$ 44 & 387 $\pm$ 89 & -10 $\pm$ 32 & -92 $\pm$ 1 & -59 $\pm$ 13 & -79 $\pm$ 5 \\
                            & medrep & \textbf{362 $\pm$ 87} & 181 $\pm$ 88 & 325 $\pm$ 184 & -68 $\pm$ 8 & -93 $\pm$ 1 & -80 $\pm$ 3 & -76 $\pm$ 4 \\
                            & medium & \textbf{408 $\pm$ 158} & 285 $\pm$ 153 & 313 $\pm$ 108 & 180 $\pm$ 101 & -83 $\pm$ 1 & -67 $\pm$ 9 & -67 $\pm$ 6 \\
\bottomrule
\end{tabular}
\end{table*}

\begin{figure}[tbp]
\vskip 0.1in
\begin{center}
\centerline{\includegraphics[width=\textwidth]{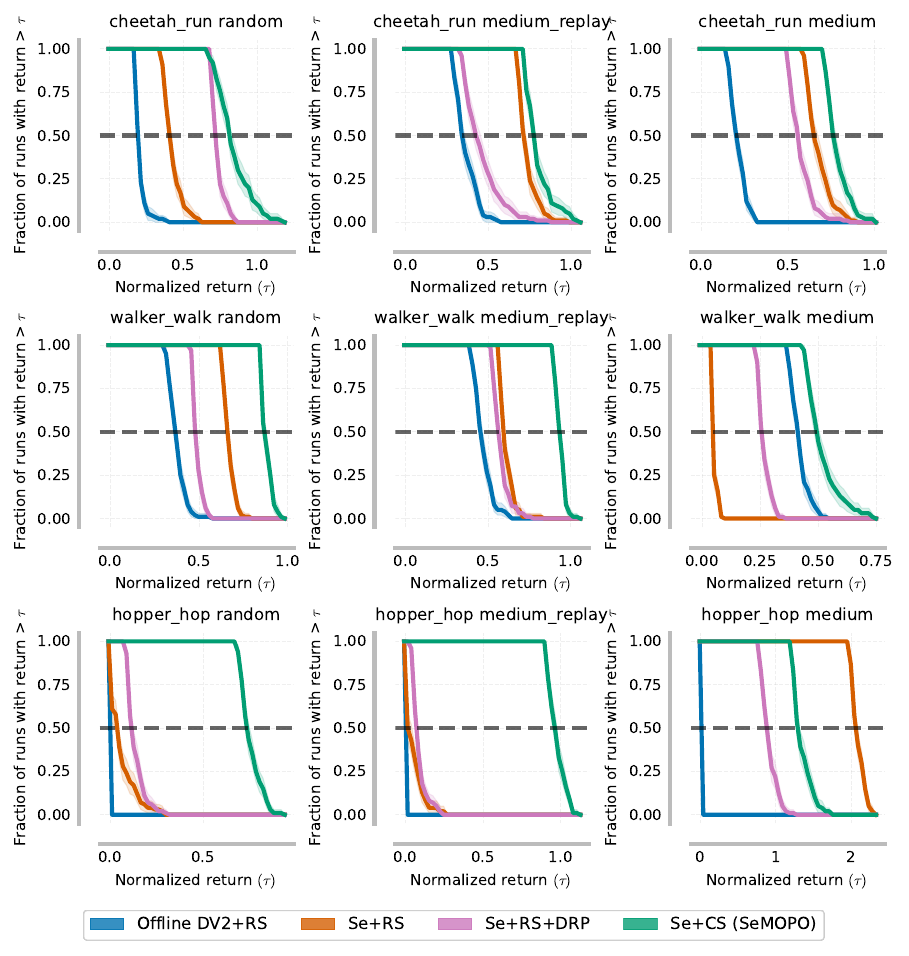}}
\caption{Performance evaluation results of ablated methods of SeMOPO on the LQV-D4RL benchmark for 200 test episodes. Shaded regions represent pointwise 95\% confidence bands based on percentile bootstrap with stratified sampling~\cite{Agarwal2021DeepRL}. Removing any component of SeMOPO leads to a performance drop.}
\label{fig:ablation_all}
\end{center}
\vskip -0.1in
\end{figure}

%%%%%%%%%%%%%%%%%%%%%%%%%%%%%%%%%%%%%%%%%%%%%%%%%%%%%%%%%%%%%%%%%%%%%%%%%%%%%%%
%%%%%%%%%%%%%%%%%%%%%%%%%%%%%%%%%%%%%%%%%%%%%%%%%%%%%%%%%%%%%%%%%%%%%%%%%%%%%%%

\end{document}